\documentclass{article}

\usepackage{microtype}
\usepackage{graphicx}
\usepackage{subcaption}
\usepackage{booktabs} % for professional tables

\usepackage{hyperref}

\usepackage[preprint]{icml2026}

\usepackage{amsmath}
\usepackage{amssymb}
\usepackage{mathtools}
\usepackage{amsthm}

\usepackage{booktabs}
\usepackage{multirow}
\usepackage{tabularx}
\usepackage{adjustbox}
\usepackage{subcaption}

\usepackage[capitalize,noabbrev]{cleveref}

\theoremstyle{plain}
\newtheorem{theorem}{Theorem}[section]

\newtheorem{lemma}[theorem]{Lemma}

\theoremstyle{definition}

\theoremstyle{remark}

\usepackage[most]{tcolorbox}

\usepackage[textsize=tiny]{todonotes}

\icmltitlerunning{Robust Heterogeneous Analog-Digital Computing for Mixture-of-Experts Models with Theoretical Generalization Guarantees}

\begin{document}

\twocolumn[
  \icmltitle{Robust Heterogeneous Analog-Digital Computing for Mixture-of-Experts Models with Theoretical Generalization Guarantees}



  \icmlsetsymbol{equal}{*}

  \begin{icmlauthorlist}
    \icmlauthor{Mohammed Nowaz Rabbani Chowdhury}{yyy}
    \icmlauthor{Hsinyu Tsai}{comp}
    \icmlauthor{Geoffrey W. Burr}{comp}
    \icmlauthor{Kaoutar El Maghraoui}{comp}
    \icmlauthor{Liu Liu}{yyy}
    \icmlauthor{Meng Wang}{yyy}
  \end{icmlauthorlist}

  \icmlaffiliation{yyy}{Rensselaer Polytechnic Institute}
  \icmlaffiliation{comp}{IBM Research}

  \icmlcorrespondingauthor{Mohammed Nowaz Rabbani Chowdhury}{chowdm2@rpi.edu}
  \icmlcorrespondingauthor{Meng Wang}{wangm7@rpi.edu}

  \icmlkeywords{Analog Compute-In-Memory, Mixture-of-Experts, Theoretical Guarantees}

  \vskip 0.3in
]



\printAffiliationsAndNotice{}  % no special notice (required even if empty)

\begin{abstract}
  Sparse Mixture-of-Experts (MoE) models enable efficient scalability by activating only a small subset of experts per input, yet their massive parameter counts lead to substantial memory and energy inefficiency during inference. Analog in-memory computing (AIMC) offers a promising solution by eliminating frequent data movement between memory and compute units. However, mitigating hardware nonidealities of AIMC typically requires noise-aware retraining, which is infeasible for large MoE models. In this paper, we propose a retraining-free heterogeneous computation framework in which noise-sensitive experts, which are provably identifiable by their \textit{maximum neuron norm}, are computed digitally while the majority of the experts are executed on AIMC hardware. We further assign densely activated modules, such as attention layers, to digital computation due to their high noise sensitivity despite comprising a small fraction of parameters. Extensive experiments on large MoE language models, including DeepSeekMoE and OLMoE, across multiple benchmark tasks validate the robustness of our approach in maintaining accuracy under analog nonidealities.
\end{abstract}

\section{Introduction}

The sparse computation of Mixture-of-Experts (MoE) models allows efficient scaling of large language and vision models without the proportional increase of training compute \cite{shazeer2017outrageously,riquelme2021scaling,fedus2022switch,chowdhury2023patch}. For each input token, the model activates only a small subset of group of neurons in the feed forward network (FFN) modules of Transformer. Each of the group is referred to as an \textit{expert}. Due to their excellent scalability, modern MoE-based large language models (LLMs) have grown to disproportionately large parameter sizes \cite{jiang2024mixtral,guo2025deepseek,yang2025qwen3}, which in turn require substantial memory to operate and result in significant energy inefficiency while computing in digital accelerators \cite{frantar2024qmoe,chowdhury2024a,jegham2025hungry,fernandez2025energy}.

The analog in-memory computing (AIMC) provides a promising solution to the inefficiencies of MoE \cite{buchel2025efficient}. AIMC allows to compute the matrix-vector multiplication (MVM) inside the non-volatile memory (NVM) devices, thus eliminating the need for frequent movement of large expert parameters between the memory and compute units, which is the major source of energy inefficiency \cite{burr2017neuromorphic,joshi2020accurate,sebastian2020memory,fujiwara202434}. However, MVM operations on AIMC hardware are inherently approximate due to various nonidealities, such as noise introduced by digital-to-analog and analog-to-digital conversions (DAC and ADC) and inaccuracies in programming weights onto NVM devices \cite{hou2025nora}. These nonidealities can lead to substantial degradation in model performance. A common strategy to alleviate this degradation is noise-aware retraining, which improves robustness to AIMC-induced noise \cite{rasch2023hardware,buchel2025analog}. However, the large parameter counts of modern MoE models renders such retraining approaches impractical for AIMC deployment.

For a robust retraining-free deployment of large MoE models in AIMC requires a heterogeneous computing approach, where the most sensitive components (e.g., noise-sensitive experts) are computed in digital accelerators, and the rest of the model components are computed in AIMC devices. However, a systematic approach is required to identify the noise-sensitive model components for the pareto-optimal heterogeneous deployment. This raises a fundamental questions:
\begin{center}
\textit{How to identify the noise-sensitive model components for the heterogeneous Digital-AIMC deployment of MoE?}
\end{center}

\begin{figure*}[t]
    \centering
    \includegraphics[width=\textwidth]{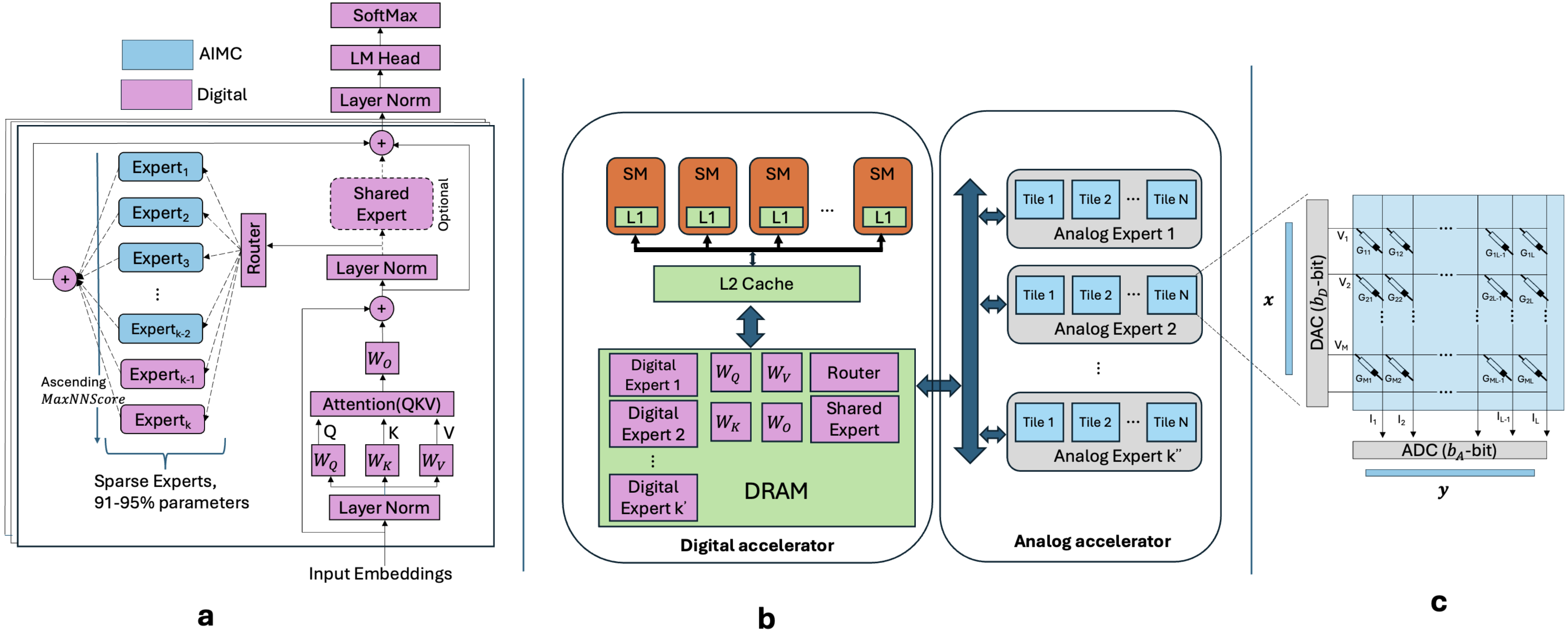}
    \caption{\textbf{a.} In heterogeneous computing of MoE, the dense modules and noise-sensitive expert modules are computed in a digital accelerator while rest of the experts are computed in an analog accelerator. \textbf{b.} A schematic of digital and analog accelerator. \textbf{c.} The analog accelerator is comprised of non-volatile memory (NVM) tiles. Weights are programmed to the crossbar array of the tile.}
    \label{fig_diag}
\end{figure*}

In this paper, we address the question both theoretically and empirically. Our theoretical analysis reveals that the experts composed of neurons with large $\ell_2$-norms are more susceptible to weight-programming noise in AIMC. We therefore propose to compute these experts in digital, while computing rest of the experts in AIMC devices. We further propose to compute the dense model components (e.g., the multi-head self-attention modules, the language modeling head, and the shared experts (if any) of MoE-LLMs) in digital due to their high per-parameter sensitivity to analog noises. We summarize our main contributions as follows:

1. \textbf{The first analog noise-sensitivity analysis of large MoE models.} To the best of our knowledge, this paper provides the first systematic sensitivity analysis of different model components in large MoE models under the two major sources of AIMC nonidealities: the DAC-ADC noise, and the weight-programming noise. Our analysis reveals that a subset of experts exhibits pronounced performance degradation as weight-programming noise increases, suggesting the computation of these experts in digital. Furthermore, we show that dense model components in MoE—despite accounting for only a small portion of the total parameters—are highly sensitive to both DAC–ADC and weight-programming noise due to their involvement in processing all input tokens.

2. \textbf{A theoretically grounded metric—the expert \textit{maximum neuron norm score}—for digital expert selection.} We theoretically prove that, the experts specializing in frequently appearing important tokens in data contain neurons with large $\ell_2$ norms, hence demonstrate high sensitivity to weight-programming noise. We therefore propose to compute the experts with large \textit{maximum neuron norm scores}—defined as the product of the maximum neuron $\ell_2$ norms across all linear projection layers within an expert—in digital.

3. \textbf{Empirical validation of the proposed heterogeneous computing approach on large MoE models.} We empirically evaluate the proposed heterogeneous computing framework on large-scale MoE models, including DeepSeekMoE \cite{dai2024deepseekmoe} (16B parameters) and OLMoE \cite{muennighoff2025olmoe} (7B parameters), across a diverse set of benchmark LLM tasks. The results demonstrate that our maximum neuron norm–based digital expert selection consistently outperforms prior expert selection strategies proposed in the MoE literature, demonstrating robustness in maintaining accuracy under imperfect analog computation.

\subsection{Related Works}

\textbf{Analog in-memory computing of large models.} The evaluation of in-memory computing performance is largely limited to smaller size models of different architectures such as CNNs \cite{joshi2020accurate,gokmen2019marriage}, RNNs \cite{rasch2023hardware}, LSTMs\cite{tsai2019inference}, and encoder only Transformers \cite{spoon2021toward}. \cite{pmlr-v235-zhang24i} minimize DAC-ADC noise for LLMs through activation-shifting and weight-reshaping. \cite{buchel2025analog} conduct large scale noise-aware retraining for foundation LLMs. \cite{hou2025nora} conduct sensitivity analysis of LLMs on various analog noise sources and proposes a training-free rescaling method to mitigate the noise. In all of the cases, the model sizes (1-3B) are well below typical MoE parameter ranges.

\textbf{Efficient inference of Mixture-of-Experts.} Due to the large memory requirements of MoE, inference cost reduction has drawn much attention recently \cite{koishekenov2023memory, li2024examining, chowdhury2024a, frantar2024qmoe, huang2025mixture}. Most of the works focus on expert pruning and quantization to mitigate inference cost. \cite{buchel2025efficient} demonstrate improved scalability of MoE in 3D AIMC architecture compared to dense models.

\textbf{Theoretical generalization guarantees of neural networks.} Due to high non-convexity, the theoretical generalization analysis of neural networks is mostly limited to two-layers \cite{chowdhury2023patch, bu2024provably, li2025task}. Feature learning dynamics is widely used to gain insights about different training and inference properties \cite{li2024nonlinear, li2025training}. Most works require simplified data model to facilitate analysis \cite{brutzkus2018sgd,karp2021local,shi2022a}.

\section{Background}

\subsection{The Mixture-of-Experts Architecture}

An MoE block replaces the dense feed-forward network (FFN) of a Transformer layer with a set of sparsely activated FFN modules, referred to as \textit{experts}. A routing network determines which experts are activated for each input token, enabling conditional computation. Common routing strategies include \emph{token-choice routing} \cite{fedus2022switch}, where each token selects its top experts, and \emph{expert-choice routing} \cite{zhou2022mixtureofexperts}, where each expert selects a subset of tokens.

For an MoE block with $k$ experts, each expert $s \in [k]$ is implemented as a two-layer MLP or gated MLP with projection matrices $\mathbf{W}_{\mathrm{up}}^{(s)} \in \mathbb{R}^{d \times m}$ and $\mathbf{W}_{\mathrm{down}}^{(s)} \in \mathbb{R}^{m \times d}$. In the gated-MLP variant, an additional gate-projection matrix $\mathbf{W}_{\mathrm{gate}}^{(s)} \in \mathbb{R}^{d \times m}$ is used. The routing network is parameterized by a routing matrix $\mathbf{\Sigma} \in \mathbb{R}^{d \times k}$.

Given an input token $\mathbf{x} \in \mathbb{R}^d$, the routing network computes routing scores $\mathbf{x}^\top\mathbf{\Sigma}$. A sparse subset of experts is activated according to the chosen routing strategy and the input $\mathbf{x}$, and the output token is computed as $\mathbf{x}_{\mathrm{out}} =
\sum_{s \in \mathcal{A}}
\mathbf{f}_s(\mathbf{x})$, where $\mathcal{A} \subseteq [k]$ denotes the set of activated experts and $\mathbf{f}_s(\cdot)$ is the expert function.

For a standard MLP expert,
\begin{equation}
\mathbf{f}_s(\mathbf{x}) =
G^{(s)}
\big[\phi(\mathbf{x}^\top\mathbf{W}_{\mathrm{up}}^{(s)})
\mathbf{W}_{\mathrm{down}}^{(s)}\big]^\top,
\end{equation}
and for a gated-MLP expert,
\begin{equation}
\mathbf{f}_s(\mathbf{x}) =
G^{(s)}
\big[\big(
\phi(\mathbf{x}^\top\mathbf{W}_{\mathrm{up}}^{(s)})
\odot
(\mathbf{x}^\top\mathbf{W}_{\mathrm{gate}}^{(s)})
\big)
\mathbf{W}_{\mathrm{down}}^{(s)}\big]^\top.
\end{equation}

Here, $\phi(\cdot)$ denotes an element-wise activation function, $\odot$ denotes the Hadamard product, and $G^{(s)}$ is the routing weight for expert $s$, obtained by normalizing routing scores as specified by the routing mechanism.

\subsection{Analog In-Memory Computing of Neural Networks}\label{aimc_descrption}

In the analog in-memory computing (AIMC) paradigm, the matrix-vector multiplication (MVM) of a linear layer is performed in non-volatile memory (NVM) devices (e.g., PCM \cite{joshi2020accurate}, ReRAM \cite{wan2022compute}), where weights are programmed as NVM cell conductances. Digital inputs are converted to analog signals via a digital-to-analog converter (DAC), and the resulting column currents, representing the analog output, are converted back to digital by an analog-to-digital converter (ADC) for subsequent digital operation.

The MVM operation in AIMC accelerators is not exact. The nonidealities arises from the weight programming noise and other system-level nonidealities, as well as the quantization noise in DAC and ADC. The programming noise and the DAC-ADC noise   dominate over other system-level noises. 

\textbf{The weight programming noise.} The weight programming noise arises from the imprecise programming of weights in NVM devices.  \cite{le202364} proposes a mathematical model to characterize this noise, as shown in  (\ref{eq_prog_noise_model}), 
\begin{equation}\label{eq_prog_noise_model}
    \hat{W}_{i,j}=W_{i,j}+\mathcal{N}(0,\sigma_{i,j}^2), \sigma_{i,j}=c_0W_{max}+\sum_{u=1}^{3}c_u\frac{|W_{i,j}|^u}{W_{max}^{u-1}}
\end{equation}
Here, $W_{i,j}$ is the $j$-th element of the column $i$ in the NVM tile, $W_{max}$ is the maximum weight-magnitude of the column. This model shows that the  noisy weight   $\hat{W}_{i,j}$ is the sum of the original weight $W_{i,j}$ and a zero-mean Gaussian Gaussian noise that depend on $W_{i,j}$  and $W_{max}$.  For example,  \cite{le202364} fits (\ref{eq_prog_noise_model}) using the data obtained from a state-of-the-art 64-core PCM-based AIMC chip, and obtain  $c_0=0.012, c_1=0.245, c_2=-0.54, c_3=0.40$ if $W_{i,j}>0.292\times W_{max}$, and $c_0=0.014, c_1=0.224, c_2=-0.72, c_3=0.952$ otherwise.

\textbf{DAC-ADC noise.} DAC noise originates from quantizing floating-point digital inputs (typically FP16) to $b_D$-bit integers representing discrete analog levels. For a fixed input range $\beta_{\mathrm{in}}$, the quantized input is
\begin{equation}\label{eq_dac}
x_q := \frac{\beta_{\mathrm{in}}}{2^{b_D-1}-1} \Big\lfloor \text{clamp}(x, -\beta_{\mathrm{in}}, \beta_{\mathrm{in}}) \frac{2^{b_D-1}-1}{\beta_{\mathrm{in}}} \Big\rceil,
\end{equation}
where $\text{clamp}(x,a,b)=\min (\max (x,a),b)$.
ADC noise arises from quantizing the continuous column currents to $b_A$-bit digital outputs. For a fixed output range $\beta_{\mathrm{out}}$, the $i$-th column output $y^{(i)}$ of the tile is quantized to
\begin{equation}\label{eq_adc}
\begin{aligned}
y_q^{(i)} &:= \text{clamp}\Big(
\frac{\beta_{\mathrm{out}}}{2^{b_A-1}-1} 
\Big\lfloor y^{(i)} \frac{2^{b_A-1}-1}{\beta_{\mathrm{out}}} \Big\rceil,
-\beta_{\mathrm{out}}, \beta_{\mathrm{out}}
\Big),\\
\beta_{\mathrm{out}} &:= \lambda \beta_{\mathrm{in}} \max(|\mathbf{W}_{:,i}|),
\end{aligned}
\end{equation}
where $\mathbf{W}_{:,i}$ are the column weights and $\lambda$ is a tunable hyperparameter.

\textbf{DAC-ADC calibration.}
In practice, we can mitigate DAC-ADC noise by calibrating the hyperparameters $\beta_{\mathrm{in}}$ in (\ref{eq_dac}) and $\lambda$ in (\ref{eq_adc}). For each NVM tile, we set $\beta_{\mathrm{in}} = \kappa \cdot \mathrm{std}(x)$, where $\mathrm{std}(x)$ is the exponential moving average of the input standard deviation over a calibration dataset. The global hyperparameters $\kappa$ and $\lambda$ are then calibrated across all tiles over the dataset.

\iffalse 
\textbf{The weight programming noise.} The weight programming noise arises from the imprecise programming of weights in NVM devices. We consider a realistic programming noise model, fitted on the data obtained from a state-of-the-art 64-core PCM-based AIMC chip \cite{le202364}. The noise model represents the noisy weights by adding a weight-magnitude depended Gaussian component to the original weight as given by (\ref{eq_prog_noise_model}).
\begin{equation}\label{eq_prog_noise_model}
    \hat{W}_{i,j}=W_{i,j}+\mathcal{N}(0,\sigma_{i,j}^2), \sigma_{i,j}=c_0W_{max}+\sum_{u=1}^{3}c_u\frac{|W_{i,j}|^u}{W_{max}^{u-1}}
\end{equation}
Here, $W_{i,j}$ is the $j$-th element of the column $i$ in the NVM tile, $W_{max}$ is the maximum weight-magnitude of the column, $c_0=0.012, c_1=0.245, c_2=-0.54, c_3=0.40$ if $W_{i,j}>0.292\times W_{max}$, and $c_0=0.014, c_1=0.224, c_2=-0.72, c_3=0.952$ otherwise.
\fi

\section{The Proposed Heterogeneous Computation of MoE}

Before we describe the our method of heterogenous computation of MoE, we first introduce a new metric that we will use to select experts.

\textbf{Maximum neuron norm score -- a metric to select experts}.  %We design a digital expert selection strategy, where the experts with larger maximum neuron norm to be computed in digital. 
For a projection matrix $\mathbf{W} \in \mathbb{R}^{d \times m}$ with $m$ neurons, we define the maximum neuron norm of this matrix $\mathbf{W}$ as 
\begin{equation}
\mathrm{MaxNNorm}(\mathbf{W}) :=
\max_{i \in [m]} \lVert \mathbf{W}_{:,i} \rVert_2,
\end{equation}
where $\mathbf{W}_{:,i} \in \mathbb{R}^d$ denotes the weight vector of neuron $i$.

For expert $s$, we define the \textit{maximum neuron norm score} of   a gated-MLP expert $s$ as the product of the maximum neuron norms of  three matrices, $\mathbf{W}_{\textrm{up}}^{(s)}$, $\mathbf{W}_{\textrm{down}}^{(s)}$, and $\mathbf{W}_{\textrm{gate}}^{(s)}$,    
\begin{equation}\label{eqn:metric}
\mathrm{MaxNNScore}^{(s)} :=
\prod_{* \in \{\mathrm{up,down,gate}\}}
\mathrm{MaxNNorm}(\mathbf{W}_*^{(s)}). 
\end{equation}
where the gate term in (\ref{eqn:metric}) is omitted when defining the maximum neuron norm score for a standard MLP expert $s$. %Within each MoE block, experts are ranked in descending order of $\mathrm{MaxNNScore}^{(s)}$, and the top $\Gamma$ fraction are selected for digital computation.

 Our   \textbf{heterogeneous computation strategy for MoE} is summarized in Figure~\ref{fig:hetero-moe}.   The idea is  to assign all densely activated modules and the top $\Gamma$ fraction of experts (ranked by the maximum neuron norm score) to digital computation, with the remaining experts’ linear modules computed on the AIMC accelerator.

\begin{figure}[t]
\centering
\begin{tcolorbox}[
    enhanced,
    colback=white,
    colframe=black,
    arc=20pt,
    boxrule=0.8pt,
    center title,
    width=0.99\linewidth,
    title=\textbf{Heterogeneous Computation of MoE}
]
\textbf{Step 1.} Place all densely activated modules on the digital accelerator.

\textbf{Step 2.} Rank experts in each MoE block by the maximum neuron norm score in (\ref{eqn:metric}) from high to low.

\textbf{Step 3.} Assign the top $\Gamma$ fraction experts to the digital accelerator. Compute the remaining experts’ linear modules on the AIMC chip.
\end{tcolorbox}
\caption{The   heterogeneous computation strategy for MoE models.  }
\label{fig:hetero-moe}
\end{figure}

\textbf{Rationale behind computing dense modules in digital.} 
The number of parameters in the densely activated modules (e.g., multi-head self-attention (MHSA), any shared expert, and LM head in MoE-LLM) is insignificant. Typically, 5-6\% parameters resides in these modules in total. On the other hand, the processing of all input tokens by these modules indicates high sensitivity to AIMC noise. Computing these modules in AIMC suggests a large drop in model performance for a minimal gain in efficiency. We therefore propose to compute them in digital.

\textbf{Rationale for computing programming-noise-sensitive experts in digital.} Unlike DAC-ADC noise, which can be mitigated via higher bit precision ($b_D$, $b_A$) and calibration of $\beta_{\mathrm{in}}$ and $\lambda$, weight programming noise varies across AIMC devices and cannot be controlled after fabrication. Selecting $\Gamma$ fraction of experts for digital computation based on their programming-noise sensitivity allows flexible trade-offs between efficiency and model performance across devices.

One major contribution of this paper is that we theoretically justify that the maximum neuron norm score is a valid metric for an expert’s sensitivity to weight-programming noise: a larger score indicates higher sensitivity, motivating placement on the digital accelerator. The formal theoretical analysis will be represented in Section \ref{sec:thm}.  

\iffalse
\textbf{The proposed digital expert selection method.} We design a digital expert selection strategy, where the experts with larger maximum neuron norm to be computed in digital. For a projection matrix $\mathbf{W} \in \mathbb{R}^{d \times m}$ with $m$ neurons, we define
\begin{equation}
\mathrm{MaxNNorm}(\mathbf{W}) :=
\max_{i \in [m]} \lVert \mathbf{W}_{:,i} \rVert_2,
\end{equation}
where $\mathbf{W}_{:,i} \in \mathbb{R}^d$ denotes the weight vector of neuron $i$.

For expert $s$, we define the \textit{maximum neuron norm score} as
\begin{equation}
\mathrm{MaxNNScore}^{(s)} :=
\prod_{* \in \{\mathrm{up,down,gate}\}}
\mathrm{MaxNNorm}(\mathbf{W}_*^{(s)}),
\end{equation}
where the gate term is omitted for standard MLP experts. Within each MoE block, experts are ranked in descending order of $\mathrm{MaxNNScore}^{(s)}$, and the top $\Gamma$ fraction are selected for digital computation.

\fi 

\section{Theoretical Support of Experts Selection}\label{sec:thm}

\subsection{Key Theoretical Insights}
To support our metric for expert selection, 
we analyze the training dynamics of a simpler   MoE model to characterize the properties of different experts and their resulting generalization.
 We consider an analytical framework where an MoE model is trained on a binary sequence classification task. The class label is determined by a task-relevant token in the sequence. There are two class-relevant tokens under each class. One of them appears more frequently than other in the sequences. We denote the frequency of the \textit{less-frequent} task-relevant token by $\alpha$, i.e., the \textit{more-frequent} task-relevant token appears with frequency $(1-\alpha)$. Before presenting formal theorems, here we summarize our key theoretical insights.

\textbf{1. Experts specialized in learning more-frequent tokens have  neurons with large magnitude.} We theoretically show that the experts which are specialized in more-frequent tokens hosts neurons with large $l_2$ norm. Therefore, these experts are more sensitive to weight-programming noise of AIMC devices.

\textbf{2. Selecting experts with large maximum neuron norm for digital computation improves noise tolerance.} Our analysis reveals that selecting experts with large maximum neuron norm, i.e., large value of $\operatorname{MaxNNScore}$ for digital computation allows the remaining analog experts to tolerate higher programming noise. Specifically, selecting a fraction of experts with top $\operatorname{MaxNNScore}$ that are specialized in more-frequent tokens allows the remaining experts to tolerate $\Omega(\frac{1-\alpha}{\alpha})$ times higher noise magnitude compared with the tolerable noise magnitude when all experts are computed in analog.

\subsection{Analytical Setup}

We adopted the same theoretical setup as in \cite{chowdhury2026efficient}. We briefly describe the setup here.

\textbf{Analytical model and the learning task.} We analyze a simplified MoE model of a single MoE block of standard MLP, trained on a binary sequence classification task. Given a sequence of tokens $\mathbf{X}\in\mathbb{R}^{d\times n}$, where each column $j\in[n]$ represents a token $\mathbf{x}^{(j)}\in\mathbb{R}^d$. The sequence label is denoted by $y\in\{+1,-1\}$. The model output is given by,
\begin{equation}\label{th_model}
    f(\mathbf{X}):=\cfrac{1}{d}\sum_{j\in[n]}\mathbf{1}^\top\mathbf{x}_{\operatorname{out}}^{(j)}
\end{equation}
where $\mathbf{x}_{\operatorname{out}}^{(j)}$ is the $j$-th output token of the MoE block and $\mathbf{1} \in \mathbb{R}^d$ is the all-ones vector. Each expert's down-projection matrix is fixed throughout the training and given by $\mathbf{W}_{\operatorname{down}}^{(s)}=a^{(s)}\mathbf{1}^{m\times d}$, where $\mathbf{1}^{m\times d}$ is an all-ones matrix and $a^{(s)}\in\{+1,-1\}$. We consider expert-choice routing, where for each expert $s$, the tokens corresponding to the top-$l$ entries of $[\mathbf{X}^\top \mathbf{\Sigma}]_{:,s}$ are routed to that expert. The routing weight of token $j$, routed to expert $s$, is evaluated as,
\begin{equation}
    G_j^{(s)}=\exp([\mathbf{X}^\top\mathbf{\Sigma}]_{j,s})/\sum_{i\in J_s(\mathbf{X})}\exp([\mathbf{X}^\top\mathbf{\Sigma}]_{i,s})
\end{equation}
where, $J_s(\mathbf{X})$ is the index set of tokens routed to expert $s$.

\textbf{Heterogeneous Computation}.
To characterize the model under heterogeneous computation, we   ignore ADC and DAC noise, as these can be largely mitigated through calibration. For weight-programming noise, we simplify the model\footnotetext{Our analysis can be extended to the full model in (\ref{eq_prog_noise_model}),  at the cost of additional calculation,    without changing the main insights. Therefore, we focus on the simplified model in the analysis.} in   (\ref{eq_prog_noise_model}) by considering only the first term in $\sigma_{i,j}$, resulting in 

\begin{equation}\label{eqn:c}
\hat{W}_{i,j}=W_{i,j}+\mathcal{N}(0,c^2W^2_{\operatorname{max}}).
\end{equation} 
(\ref{eqn:c}) allows us to vary $c$ to control the noise level.  

We consider two scenarios: (i) Fully heterogeneous computation, as described in Figure \ref{fig:hetero-moe}, with the resulting model denoted as $f_H(\mathbf{X})$; (ii) 
All-experts-in-analog computation, where linear modules of all experts are computed on the AIMC chip, denoted as $f_A(\mathbf{X})$.

\iffalse 

\textbf{Weight-programming noise model.} We consider a variable programming-weight magnitude $c$. The analyzed noise model is given by 

\begin{equation}\label{eqn:c}
\hat{W}_{i,j}=W_{i,j}+\mathcal{N}(0,c^2W^2_{\operatorname{max}}).
\end{equation}
\fi 

\textbf{Sequence sampling model.} Let $(\mathbf{X},y)\sim\mathcal{D}$, where tokens of $\mathbf{X}$ drawn from an orthonormal set $\mathcal{P} \subset \mathbb{R}^d$. Two vectors $\mathbf{o}_1, \mathbf{o}_2 \in \mathcal{P}$ define the task-relevant set $\mathcal{P}_r = \{\pm \mathbf{o}_1, \pm \mathbf{o}_2\}$; all others are task-irrelevant. Each sequence contains exactly one task-relevant token, determining the label: sequences containing $\pm \mathbf{o}_1$ are labeled $+1$, and those containing $\pm \mathbf{o}_2$ are labeled $-1$. Remaining tokens are drawn independently from $\mathcal{P}\backslash\{\mathbf{o}_1,\mathbf{o}_2\}$. With probability $\alpha \in (0,1/4)$, the task-relevant token is $\mathbf{o}_1$ or $\mathbf{o}_2$.

\textbf{Training Setup}. We analyze the case, where the network given in (\ref{th_model}) is trained for $T$ steps via SGD with batch size $B$ to minimize the Hinge loss, i.e., $\max(1-yf^{(t)}(\mathbf{X}),0)$.

\textbf{Expert specialization.} To quantify expert specialization, we define, for each expert $s$ and task-relevant token $\mathbf{v} \in \mathcal{P}_r$, the probability
\begin{equation}
p_{\mathbf{v}}^{(s)} := \mathbb{P}\Big[G_j^{(s,t)} \ge 1/l \;\big|\; x^{(j)} = \mathbf{v} \text{ for some } j \in [n] \Big]
\end{equation}
over the data distribution $\mathcal{D}$. A value of $p_{\mathbf{v}}^{(s,t)} = 1$ indicates that the task-relevant token $\mathbf{v}$ is routed to expert $s$ in every sequence containing $\mathbf{v}$, i.e., expert $s$ is fully specialized on $\mathbf{v}$.

\subsection{Theoretical Generalization Guarantees}

\begin{lemma}\label{lm_m_1}
    Suppose, the model in (\ref{th_model}) is trained for
    $T=\Theta(l^2\sqrt{\log l}/\alpha)$ steps. For any $\mathbf{v}\in\{\mathbf{o}_1,\mathbf{o}_2\}$ and any expert $s,s^\prime$, such that $p_{\mathbf{v}}^{(s,T)}=1$ and $p_{\mathbf{-v}}^{(s^\prime,T)}=1$, we have
    \begin{equation*}
        \operatorname{MaxNNScore}^{(s)}<\operatorname{MaxNNScore}^{(s^\prime)}
    \end{equation*}
\end{lemma}

Lemma \ref{lm_m_1} shows that, the expert specialized on more-frequent task-relevant token (e.g., the expert $s^\prime$ with $p_{-\mathbf{o}_1}^{(s^\prime)}=1$) maintains larger \textit{maximum neuron norm score} compared to the expert specialized on less-frequent task-relevant token (e.g., the expert $s$ with $p_{\mathbf{o}_1}^{(s^\prime)}=1$).

The weight-programming noise model in~(\ref{eq_prog_noise_model}) indicates that larger weight magnitudes incur higher analog noise, while also corresponding to larger neuron norms. Based on this intuition, we establish formal generalization guarantees for both fully analog and heterogeneous computation in Theorem~\ref{th_1}.

\begin{table*}[t]
\centering
\small
\setlength{\tabcolsep}{2pt}
\renewcommand{\arraystretch}{1.2}
\caption{Accuracy (\%) of MoE models with DAC--ADC noise added to the inputs and outputs of different modules.}
\label{tab_adc}
\begin{tabularx}{\textwidth}{lll*{8}{>{\centering\arraybackslash}X}p{0.9cm}}
\toprule
\textbf{Model} & \textbf{Noise} & \textbf{Modules} 
& \multicolumn{8}{c}{\textbf{Tasks}} 
& \textbf{Avg.} \\
\cmidrule(lr){4-11}
& & 
& \textbf{PIQA} & \textbf{ARC-e} & \textbf{ARC-c} & \textbf{BoolQ} 
& \textbf{HellaS.} & \textbf{Wino.} & \textbf{MathQA} & \textbf{MMLU} & \\
\midrule
\multirow{3}{*}{\shortstack{DeepSeek-\\MoE}}
& Digital (FP-16) & --- 
& 80.36 & 72.85 & 47.61 & 72.78 & 77.38 & 70.17 & 31.42 & 37.67 & 61.28 \\
\cmidrule(lr){2-12}
& \multirow{2}{*}{DAC-ADC} & Experts
& 79.16 & 73.40 & 47.44 & 72.57 & 76.78 & 69.77 & 31.22 & 37.92 & \bf{61.03} \\
& & Experts+Dense
& 75.68 & 68.06 & 39.16 & 68.29 & 68.58 & 60.69 & 28.17 & 31.71 & 55.04 \\
\midrule
\multirow{3}{*}{OLMoE}
& Digital (FP-16) & --- 
& 79.71 & 76.94 & 49.66 & 69.85 & 78.19 & 68.90 & 28.44 & 53.65 & 63.17 \\
\cmidrule(lr){2-12}
& \multirow{2}{*}{DAC-ADC} & Experts
& 79.49 & 76.43 & 49.49 & 69.02 & 77.34 & 67.72 & 28.98 & 51.22 & \bf{62.46} \\
& & Experts+Dense
& 77.86 & 74.03 & 48.89 & 68.87 & 77.57 & 67.88 & 27.81 & 50.43 & 61.54 \\
\bottomrule
\end{tabularx}
\end{table*}

\begin{figure*}[t]
    \centering
    \begin{subfigure}[t]{0.40\textwidth}
        \centering
        \includegraphics[width=0.9\linewidth]{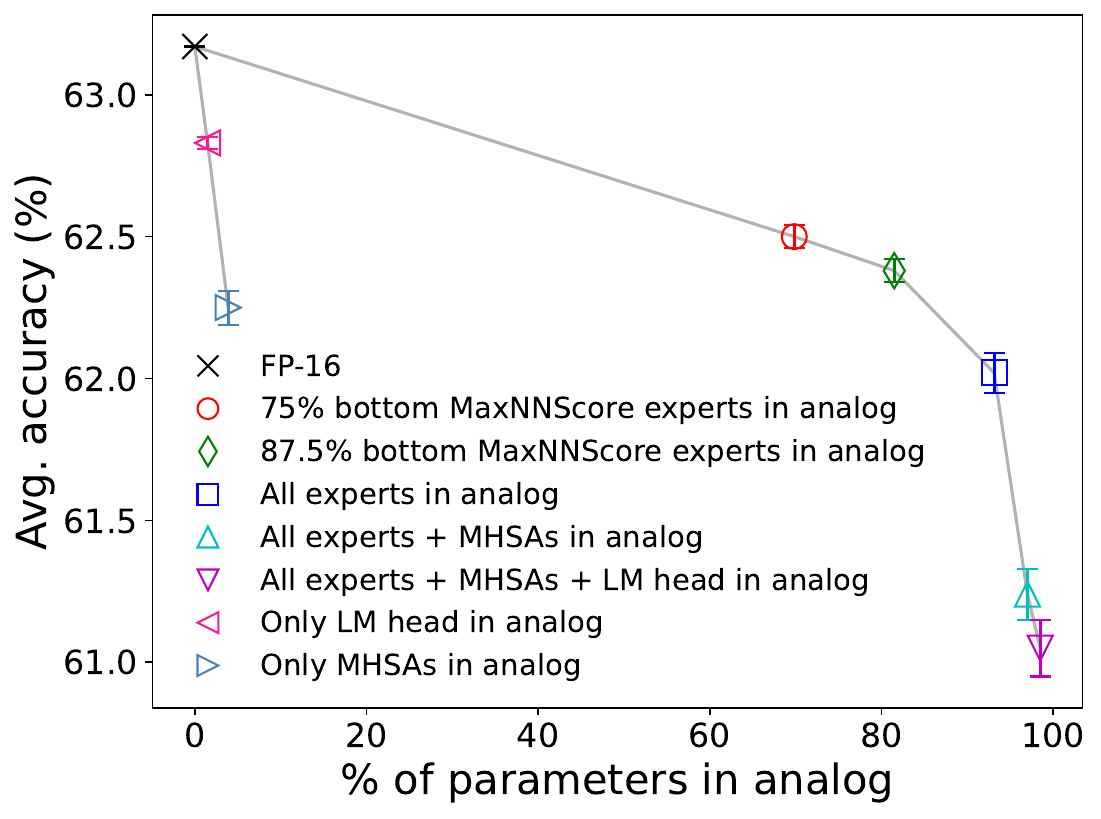}
        \caption{OLMoE}
        \label{fig_dense_olmoe}
    \end{subfigure}
    \hspace{1cm}
    \begin{subfigure}[t]{0.40\textwidth}
        \centering
        \includegraphics[width=0.9\linewidth]{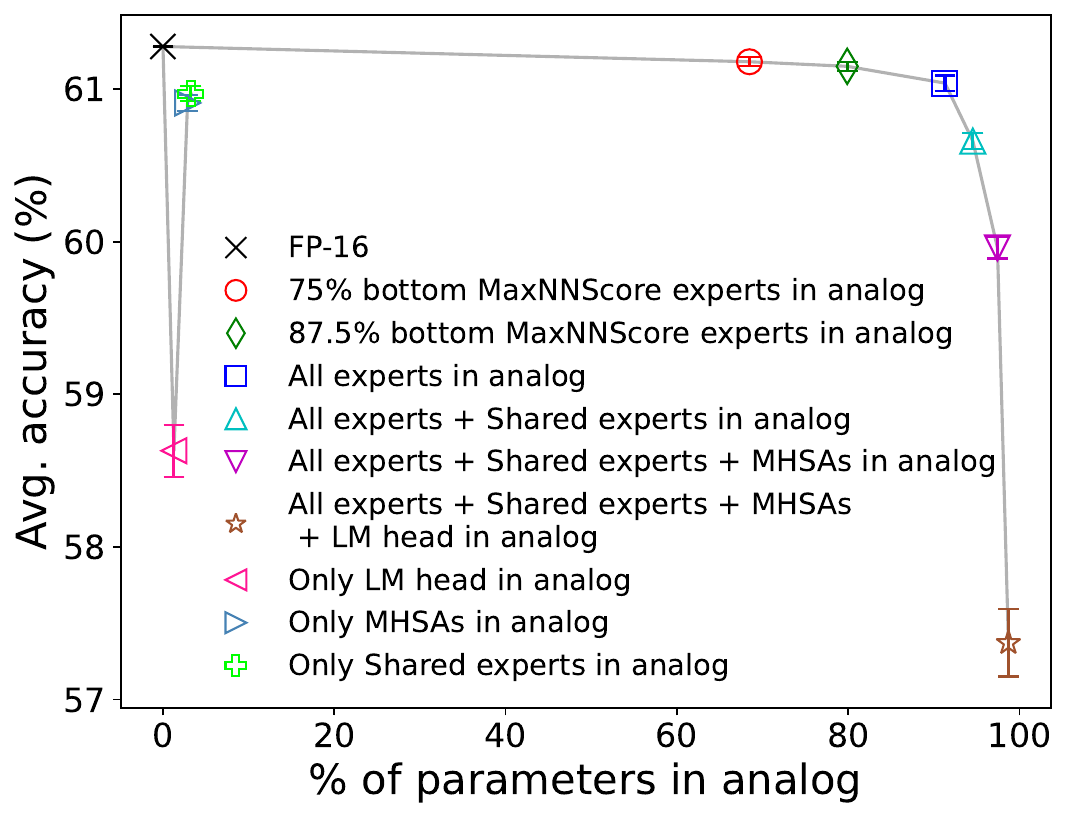}
        \caption{DeepSeekMoE}
        \label{fig_dense_deepseek}
    \end{subfigure}

    \caption{Effect of computing dense modules in analog. Accuracy degradation is significant when densely activated modules are executed on AIMC, despite their small parameter footprint.}
    \label{fig_dense}
\end{figure*}

%We denote by $f_A(\mathbf{X})$ the output of (\ref{th_model}) when all experts are computed in AIMC, with programming noise magnitude $c_{A}$. Under the proposed heterogeneous compute scheme, the output is denoted by $f_H(\mathbf{X})$, with corresponding programming noise magnitude $c_{H}$.

\begin{theorem}\label{th_1}
    Suppose the model given in (\ref{th_model}) is trained for $T=\Theta(l^2\sqrt{\log l}/\alpha)$ steps.  Let $\gamma$ denote the fraction of the experts $s\in[k]$ in the trained model that satisfies $p_{-\mathbf{v}}^{(s,T)}=1$ for $v\in\{\mathbf{o}_1,\mathbf{o}_2\}$. If the programming-weight magnitude $c$ in the noise model (\ref{eqn:c}) satisfies 
    \begin{equation}\label{c_a}
       c\leq  c_{A}:=O(\cfrac{\alpha}{1-\alpha}\times\cfrac{1}{l^2\sqrt{d\log (kmld^2)}}),
    \end{equation}
    then with high probability, the Analog-Expert MoE model, denoted by $f_{A}$,  has guaranteed generalization i.e., 
    \begin{equation}
    \mathbb{P}[\forall(\mathbf{X},y)\sim\mathcal{D}:yf_{A}^{(T)}(\mathbf{X})>0]=1.
    \end{equation}

Moreover, if $\Gamma\ge\gamma$ fraction of the experts with top $\operatorname{MaxNNScore}$ are computed in digital while others in analog, and the noise magnitude $c$ satisfies 
    \begin{equation}\label{c_h}
      % c_{H}=O(\cfrac{1}{l^2\sqrt{d\log (kmld^2)}}),、
      c \leq c_{H}:=  \frac{1-\alpha}{\alpha}c_{A},
    \end{equation}
    then with high probability,  the heterogeneous-expert MoE model, denoted by $f_{H}$,  has guaranteed generalization. i.e.,  
    \begin{equation}
    \mathbb{P}[\forall(\mathbf{X},y)\sim\mathcal{D}:yf_{H}^{(T)}(\mathbf{X})>0]=1.
    \end{equation}
\end{theorem}

Theorem~\ref{th_1} characterizes the maximum tolerable programming noise magnitude for guaranteed generalization under the fully analog and the proposed heterogeneous compute schemes, given in (\ref{c_a}) and~(\ref{c_h}), respectively. It shows that computing a fraction $\gamma$ of experts in digital, i.e., those specialized on more frequent task-relevant tokens, allows the remaining analog experts to tolerate a programming noise magnitude larger by a factor of $\Omega\!\left(\frac{1-\alpha}{\alpha}\right)$.

\begin{figure*}[t]
    \centering
    \begin{subfigure}[t]{0.40\textwidth}
        \centering
        \includegraphics[width=0.9\linewidth]{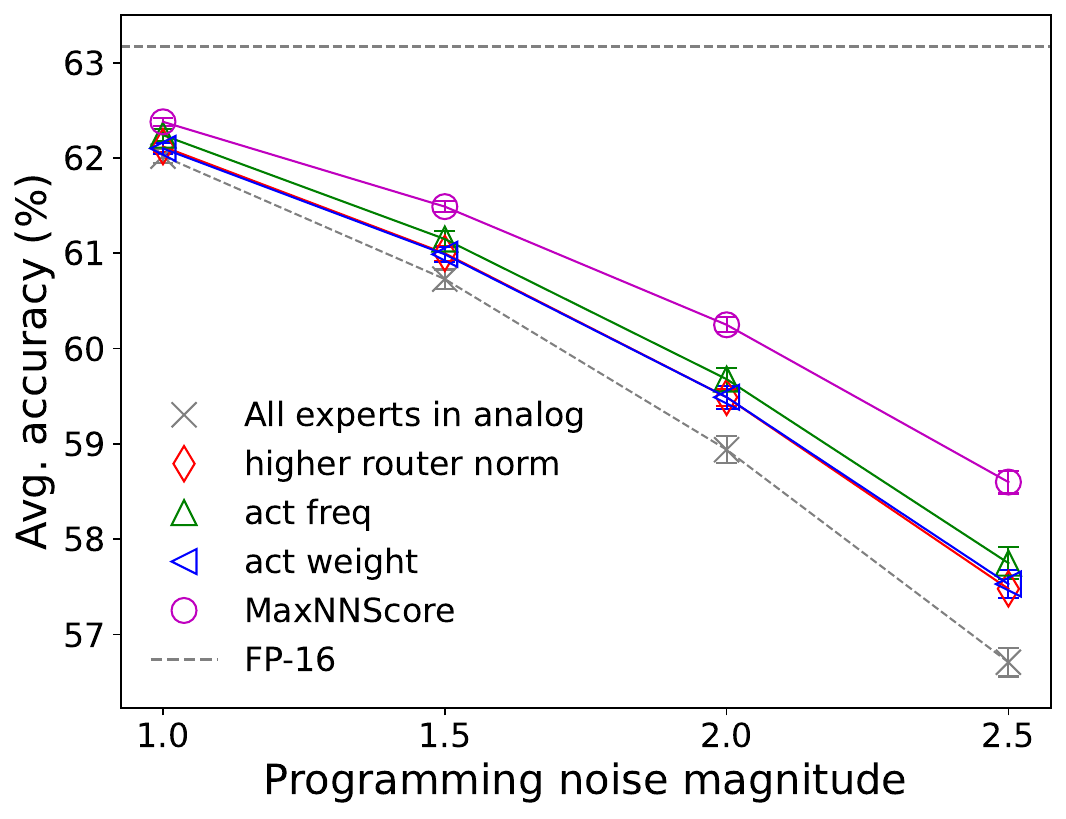}
        \caption{12.5\% experts in digital}
        \label{fig_12p5_olmoe}
    \end{subfigure}
    \hspace{1cm}
    \begin{subfigure}[t]{0.40\textwidth}
        \centering
        \includegraphics[width=0.9\linewidth]{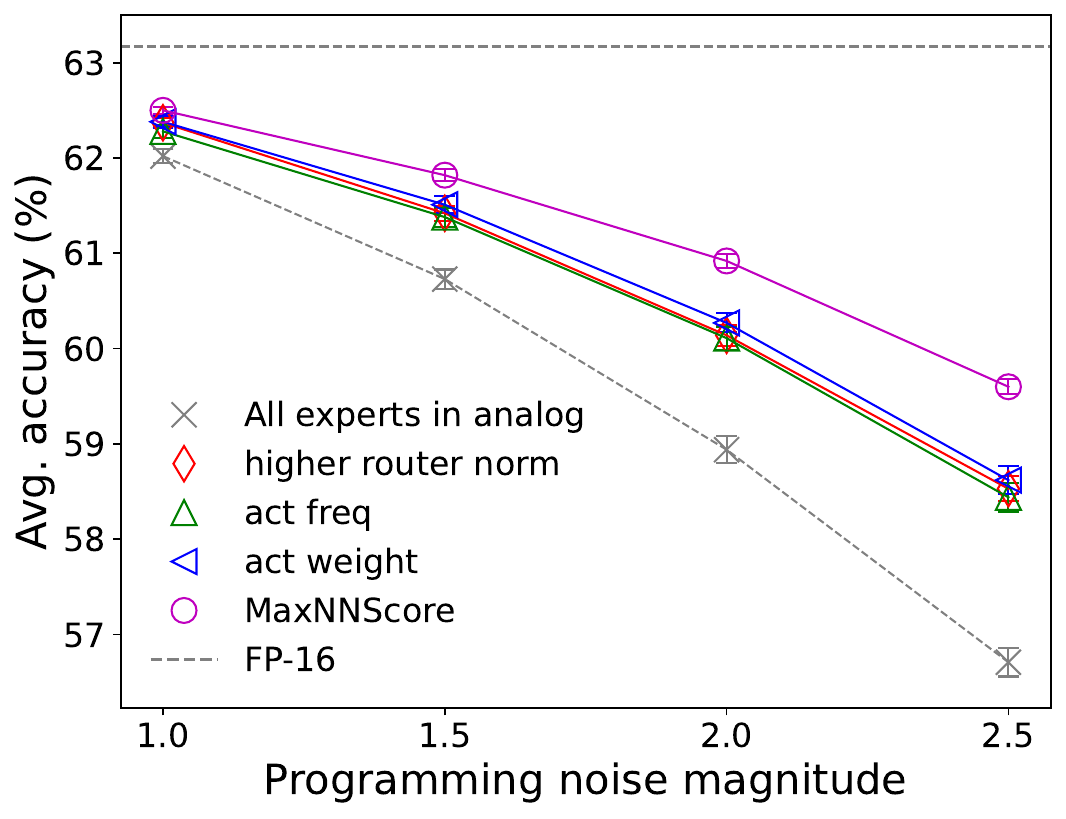}
        \caption{25\% experts in digital}
        \label{fig_25p0_olmoe}
    \end{subfigure}

    \caption{Performance of different digital expert selection methods in OLMoE}
    \label{fig_maxn_olmoe}
\end{figure*}

\begin{figure*}[t]
    \centering
    \begin{subfigure}[t]{0.40\textwidth}
        \centering
        \includegraphics[width=0.9\linewidth]{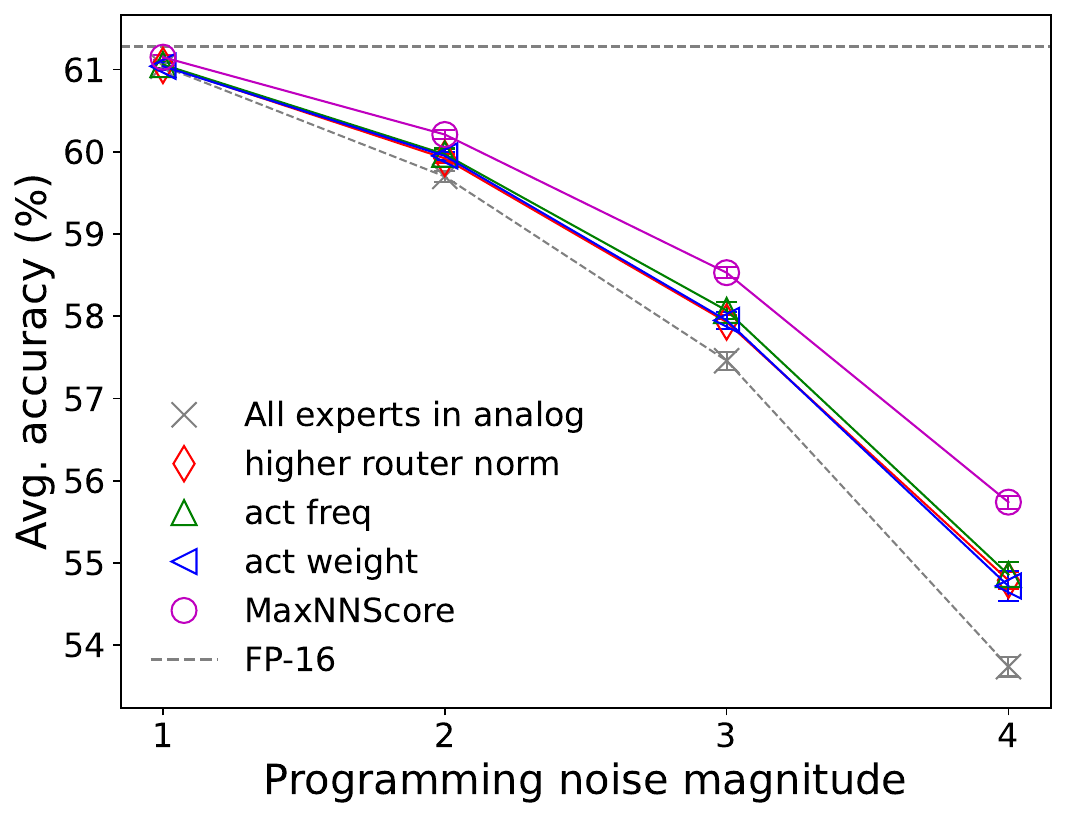}
        \caption{12.5\% experts in digital}
        \label{fig_12p5_deepseek}
    \end{subfigure}
    \hspace{1cm}
    \begin{subfigure}[t]{0.40\textwidth}
        \centering
        \includegraphics[width=0.9\linewidth]{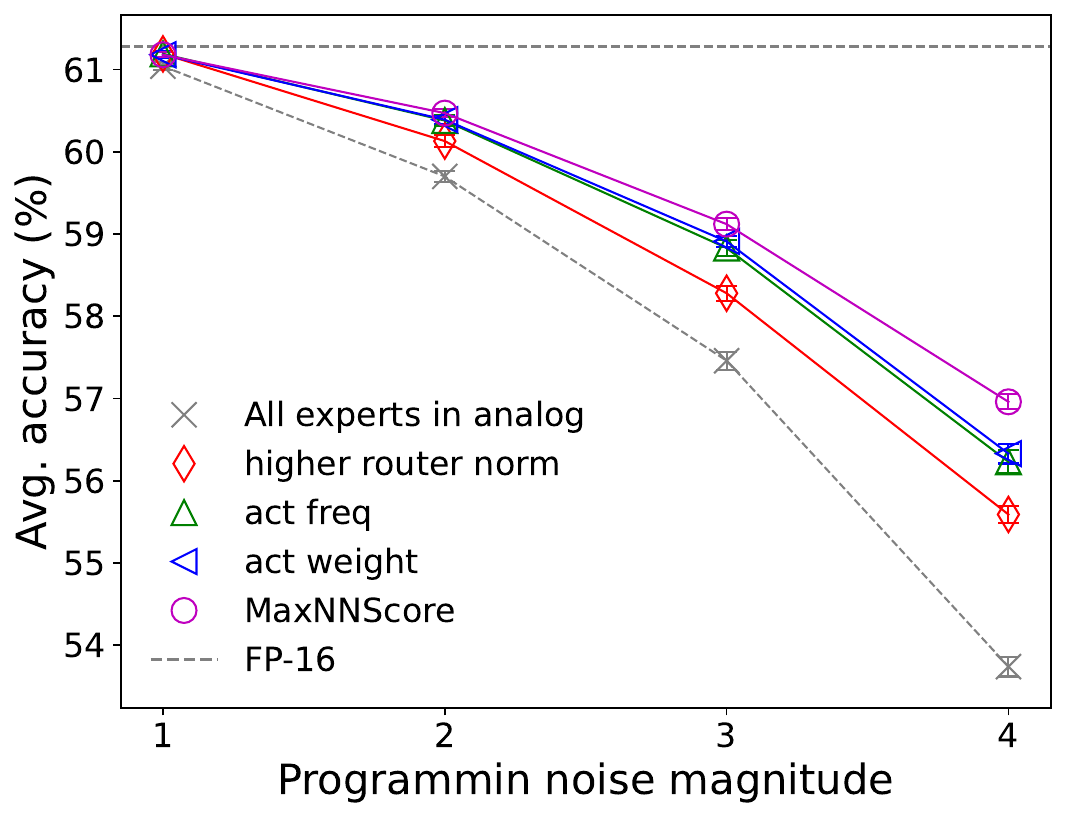}
        \caption{25\% experts in digital}
        \label{fig_25p0_deepseek}
    \end{subfigure}

    \caption{Performance of different digital expert selection methods in DeepSeekMoE}
    \label{fig_maxn_deepseek}
\end{figure*}

\section{Experiments}

\subsection{Experimental Setup}

\textbf{Evaluated models.} We evaluate the robustness of the proposed heterogeneous computation on two large pre-trained MoE LLMs: DeepSeekMoE~\cite{dai2024deepseekmoe} and OLMoE~\cite{muennighoff2025olmoe}, with 16B and 7B parameters, respectively. DeepSeekMoE consists of 28 Transformer layers, where the FFN of the first layer is dense and all subsequent FFN blocks are MoE; each MoE block includes a dense shared expert that processes all tokens. OLMoE contains 16 Transformer layers, all with MoE-based FFNs. In both models, each MoE block has 64 sparse experts.

\textbf{Evaluated task.} We test the performance of these models on 8 benchmark LLM tasks: PIQA \citep{Bisk2020}, ARC-Challenge and ARC-Easy \citep{allenai:arc}, BoolQ \citep{clark2019boolq}, HellaSwag \citep{zellers2019hellaswag}, WinoGrande \citep{sakaguchi2021winogrande}, MathQA \citep{amini-etal-2019-mathqa}, and MMLU \citep{hendryckstest2021}. These tasks cover diverse areas of LLM capabilities, spanning problem solving in various professional domains, along with commonsense and mathematical reasoning, trivia, and natural language inference.

\textbf{Analog noise simulation.} We simulate the DAC-ADC noise through IBM's AIHWKIT-Lightning \cite{aihwkitlightning}. The weight-programming noise is simulated by directly adding Gaussian noise according to (\ref{eq_prog_noise_model}). For both of the cases, we select the NVM tile size of 512. For programming noise, we report the average and standard error of the results of 32 different random seeds.

\subsection{Results on DAC-ADC Noise}

We evaluate the robustness of MoE models against DAC-ADC noise. We select 8-bit DAC and ADC for the experiments. Table \ref{tab_adc} provides the results of inference accuracy of the two MoE models across eight benchmark LLM tasks, where the DAC-ADC noise is added to the inputs and outputs of different modules. While the models show robustness for noise added to all the experts, adding noise to dense modules significantly degrade performance. The insignificant drop of accuracy (only 0.25\% in DeepSeekMoE, and 0.70\% drop in OLMoE) for the expert-only case implies that the DAC-ADC noise of the experts can be mitigated to a very low level by properly calibrating the hyperparameters of DAC and ADC. This justifies the design of our digital expert selection strategy based on weight-programming noise.

Given that calibrated DAC–ADC noise is negligible and its simulation is prohibitively slow, we report only weight-programming noise results in the remainder of this section.

\begin{table*}[t]
\centering
\small
\setlength{\tabcolsep}{4pt}
\renewcommand{\arraystretch}{1.3}
\caption{The throughput, energy efficiency, and accuracy of OLMoE evaluated for batch size 32.}
\label{tab_throughput}
\begin{tabularx}{\textwidth}{l l r r @{\extracolsep{\fill}} c c c}
\toprule
\textbf{Param.} & \textbf{Modules} & \textbf{Throughput} & \textbf{Energy Efficiency} & \multicolumn{3}{c}{\textbf{Avg. Accuracy (\%) vs. Prog. noise magnitude}} \\
\cmidrule(lr){5-7}
\textbf{in Digital} & \textbf{in Digital} & \textbf{(Tokens/s)} & \textbf{(Tokens/Watt $\cdot$ s)} & \textbf{1.0} & \textbf{1.5} & \textbf{2.5} \\
\midrule
100\% (FP-16) & --- & 4220.07 & 10.55 & \multicolumn{3}{c}{63.17} \\
\midrule
0\% (analog) & None & 768.41 & 23949.07 & 61.05 $\pm$ 0.10 & 58.45 $\pm$ 0.15 & 49.74 $\pm$ 0.23 \\
\cmidrule(lr){1-7}
5.37\% (het.gen) & Dense & 49781.23 & 123.92 & 62.02 $\pm$ 0.07 & 60.73 $\pm$ 0.10 & 56.71 $\pm$ 0.15 \\
\cmidrule(lr){2-7}
17.01\% (het.gen) & \shortstack[l]{Dense +\\ 12.5\% experts} & 24924.12 & 62.20 & 62.38 $\pm$ 0.04 & 61.49 $\pm$ 0.06 & 58.60 $\pm$ 0.12 \\
\cmidrule(lr){2-7}
28.65\% (het.gen) & \shortstack[l]{Dense +\\ 25.0\% experts} & 14513.52 & 36.25 & 62.50 $\pm$ 0.04 & 61.82 $\pm$ 0.06 & 59.60 $\pm$ 0.08 \\
\bottomrule
\end{tabularx}
\end{table*}

\subsection{Results on Weight-programming Noise}

\textbf{Effect of computing dense modules in analog.} We investigate the relative sensitivity of the dense modules and the sparse expert modules for DeepSeekMoE and OLMoE. The dense modules in OLMoE includes the MHSA modules and the LM head. The DeepSeekMoE includes the additional dense modules of shared expert.

Figure \ref{fig_dense} presents the average accuracy of OLMoE and DeepSeekMoE across the eight benchmark LLM tasks, while the standard weight-programming noise given in (\ref{eq_prog_noise_model}) is added to different dense modules separately and together with the sparse expert modules. As shown in the figure, despite having a very small percentage of parameters, all of the dense modules shows high sensitivity to the programming noise. In OLMoE, the multi-head self-attention (MHSA) modules account for only 3.88\% of the total parameters, yet placing only these modules in analog results in a larger performance degradation than placing 87.5\% lower $\operatorname{MaxNNScore}$ experts in analog, which comprise 81.5\% of the model parameters. Similarly, in DeepSeekMoE, the LM head, the MHSAs, and the shared experts account for 1.28\%, 2.87\%, and 3.26\% of the total parameters, respectively. Yet, placing each of these modules individually in analog results in higher drop than placing 100\% of the sparse experts, comprising 91.28\% of the total parameters, in analog.

\textbf{Verifying the effectiveness of the digital expert selection strategy.} We compare the performance of our proposed $\operatorname{MaxNNScore}$ based digital expert selection strategy against several baselines used for selecting expert for pruning \cite{koishekenov2023memory,chowdhury2024a} and higher quantized precision \cite{li2024examining,huang2025mixture}. We briefly describe the baselines as given below.

\textit{Activation Frequency.} This metric rank the experts based on their activation frequency, i.e., the fraction of tokens received by each expert over calibration dataset.

\textit{Activation Weight.} This metric rank the experts based on their average routing weight over calibration dataset.

\textit{Router norm.} This metric does not require any calibration data. The metric rank the experts based on the norm of their corresponding routing parameters.

Figures~\ref{fig_maxn_olmoe} and~\ref{fig_maxn_deepseek} show the average accuracy over the eight benchmark LLM tasks as a function of programming noise magnitude for different expert selection methods, evaluated on OLMoE and DeepSeekMoE, respectively. As shown in the figures, the proposed $\operatorname{MaxNNScore}$ based consistently outperforms other baselines, maintaining an increasing gap of accuracy with the noise magnitude. It can be inferred from the figures that at least one-third of the performance drop can be recovered by placing only one-eighth of the experts in digital, and half of the performance drop can be recovered by placing only one-fourth of the experts in digital.

\subsection{Energy Efficiency vs. Throughput vs. Accuracy Tradeoff in Heterogeneous Computation}

We evaluate the throughput (Tokens/s), and the energy efficiency (Tokens/Watt$\cdot$s) for the proposed heterogeneous computing method, and compare with the full analog and full digital compute paradigm. The details about the evaluation of these values are provided in Appendix \ref{dtl_thpt}.

\textbf{Heterogeneous computing provides better tradeoff.} Table \ref{tab_throughput} presents the throughput, energy efficiency, and the average accuracy over the eight benchmark tasks under different weight-programming magnitudes for OLMoE. From the table, we can infer that, the full-digital (FP-16) computing is severely energy inefficient, while providing a moderate throughput value. On the other hand, while the full-analog approach's energy efficiency is incomparably high, it's generating throughput is the lowest. Note that, unlike digital accelerators, the throughput of full-analog does not increases with batch size, which indicates higher latency for larger batch size. As expected, the approach also suffers from the worst accuracy degradation. However, as observed in the table, the proposed heterogeneous approach provides a balance between throughput and energy efficiency. Moreover, for higher noise magnitudes, computing more experts in digital will allow to trading off some compute efficiency (i.e., throughput and energy efficiency) to achieve higher model performance.

\section{Conclusion}

In this paper, a heterogeneous computing approach of MoE is provide
where the noise sensitive components are computed in digital accelerator, while rest of the components are computed in analog in-memory devices. We demonstrate that computing dense modules in digital along with the experts of high neuron norm improves model's robustness against AIMC noises. Furthermore, we show that by varying the experts' fraction in digital allows flexibility in balancing throughput, energy consumption and accuracy. Future works includes system design for dynamic computation of experts in AIMC and digital accelerators based on the compute and energy budget.   

%\clearpage
\section*{Impact Statement}

%This paper presents work whose goal is to advance the field of Machine
%Learning. There are many potential societal consequences of our work, none
%

This work aims to improve the efficiency and robustness of deploying large Mixture-of-Experts models through heterogeneous analog–digital computing. By reducing memory movement and energy consumption while maintaining accuracy, the proposed approach may lower the computational and environmental costs of large-scale model deployment and enable use in more resource-constrained settings. The work does not  directly address or exacerbate societal concerns such as bias, fairness, or misuse.

\bibliography{ref}
\bibliographystyle{icml2026}

%%%%%%%%%%%%%%%%%%%%%%%%%%%%%%%%%%%%%%%%%%%%%%%%%%%%%%%%%%%%%%%%%%%%%%%%%%%%%%%
%%%%%%%%%%%%%%%%%%%%%%%%%%%%%%%%%%%%%%%%%%%%%%%%%%%%%%%%%%%%%%%%%%%%%%%%%%%%%%%
% APPENDIX
%%%%%%%%%%%%%%%%%%%%%%%%%%%%%%%%%%%%%%%%%%%%%%%%%%%%%%%%%%%%%%%%%%%%%%%%%%%%%%%
%%%%%%%%%%%%%%%%%%%%%%%%%%%%%%%%%%%%%%%%%%%%%%%%%%%%%%%%%%%%%%%%%%%%%%%%%%%%%%%
\newpage
\appendix
\onecolumn
\section{The Details for Evaluating Throughput and Energy Efficiency}\label{dtl_thpt}

To evaluate the two quantities, we consider that the digital accelerator is equivalent to an NVIDIA A100 GPU. The quantities are evaluated under 100\% model FLOPs utilization (MFU), with 624 tera operations per second (624 TOP/s) at 400Watt, and with data transfer bandwidth of 1,555 GB/s for the accelerator. In that case, the throughput is evaluated as,
\begin{equation}
\begin{aligned}
    &\text{Throughput (tokens/s)}
    &=\cfrac{\text{No. of tokens generated}}{\max(\frac{\text{Total TOPs}}{\text{MFU624TOPs/s}}, \frac{\text{Total weight transfer}}{1555\text{GB/s}})}
\end{aligned}
\end{equation}
The energy efficiency can be directly calculated using the throughput and power rating of the device.

For the analog accelerator, we compute the throughput by dividing the total number of tokens generated during inference by the total latency accumulated over all the asynchronous operations in forward passes. Similarly, the energy efficiency is evaluated by dividing the number of tokens generated by the total energy consumption accumulated over all operations in forward passes. The latencies and energy consumption values for different operations are taken from \cite{buchel2025efficient}.

For the heterogeneous case, we take the upper-bound of the latencies over both accelerators, and the energy consumption is evaluated by adding the product of the digital accelerator's power rating and latency to analog accelerator's energy consumption.

\section{Hyperparameter Calibration Results for DAC and ADC}

We calibrate the hyperparameter $\kappa$ and $\lambda$ to obtain the results in Table \ref{tab_adc}, according to the method described in section \ref{aimc_descrption}. We use Wkitext-103 test set for calibration. The results are given in

\begin{table}[h]
    \centering
    \caption{$\kappa$ vs. Perplexity (PPL) for OLMoE. Noise added only at the experts.}
    \label{tab:kappa_ppl}
    \begin{tabular}{c|cccccccc}
        \hline
        $\kappa$ & 10 & 18 & 25 & 30 & \bf{35} & 40 & 45 & 50 \\
        \hline
        PPL & 10.98 & 7.59 & 7.10 & 7.00 & \bf{6.97} & 6.99 & 7.02 & 7.07 \\
        \hline
    \end{tabular}
\end{table}

\begin{table}[h]
    \centering
    \caption{$\lambda$ vs. Perplexity (PPL) for OLMoE ($\kappa=35$). Noise added only at the experts.}
    \label{tab:lambda_ppl}
    \begin{tabular}{c|ccccccccccc}
        \hline
        $\lambda$ & 0.75 & 0.9 & \bf{1.0} & 1.125 & 1.25 & 1.5 & 1.75 & 2.0 & 2.25 & 2.5 & 2.75 \\
        \hline
        PPL & 7.15 & 7.10 & \bf{7.10} & 7.11 & 7.13 & 7.19 & 7.26 & 7.32 & 7.42 & 7.48 & 7.57 \\
        \hline
    \end{tabular}
\end{table}

\begin{table}[h]
    \centering
    \caption{$\kappa$ vs. Perplexity (PPL) for OLMoE. Noise added to both experts and dense modules.}
    \label{tab:kappa_ppl_full}
    \begin{tabular}{c|cccccccc}
        \hline
        $\kappa$ & 10 & 15 & 20 & 25 & \bf{30} & 35 & 40 & 45 \\
        \hline
        PPL & 22.26 & 8.66 & 7.50 & 7.16 & \bf{7.07} & 7.07 & 7.12 & 7.18 \\
        \hline
    \end{tabular}
\end{table}

\begin{table}[h]
    \centering
    \caption{$\lambda$ vs. Perplexity (PPL) for OLMoE ($\kappa=30$). Noise added to experts and dense modules.}
    \label{tab:lambda_ppl_full}
    \begin{tabular}{c|cccc}
        \hline
        $\lambda$ & 0.75 & \bf{1.0} & 1.25 & 1.5 \\
        \hline
        PPL & 7.42 & \bf{7.26} & 7.28 & 7.38 \\
        \hline
    \end{tabular}
\end{table}

\begin{table}[h]
    \centering
    \caption{$\kappa$ vs. Perplexity (PPL) for DeepSeekMoE. Noise added to experts.}
    \label{tab:kappa_ppl_additional}
    \begin{tabular}{c|cccccc}
        \hline
        $\kappa$ & 30 & 35 & \bf{40} & 45 & 50 & 60 \\
        \hline
        PPL & 6.64 & 6.62 & 6.62 & 6.63 & 6.64 & 6.68 \\
        \hline
    \end{tabular}
\end{table}

\begin{table}[h]
    \centering
    \caption{$\lambda$ vs. Perplexity (PPL) for DeepSeekMoE ($\kappa=40$). Noise added to experts only.}
    \label{tab:lambda_ppl_additional}
    \begin{tabular}{c|ccccccccc}
        \hline
        $\lambda$ & 0.9 & 1.0 & 1.25 & 1.5 & 2.0 & \bf{2.25} & 2.5 & 3.0 & 4.0 \\
        \hline
        PPL & 6.97 & 6.86 & 6.73 & 6.69 & 6.68 & 6.67 & 6.67 & 6.69 & 6.76 \\
        \hline
    \end{tabular}
\end{table}

\begin{table}[h]
    \centering
    \caption{$\kappa$ vs. Perplexity (PPL) for DeepSeekMoE. Noise added to experts and dense modules.}
    \label{tab:kappa_ppl_simple}
    \begin{tabular}{c|ccccc}
        \hline
        $\kappa$ & 30 & \bf{35} & 40 & 45 & 50 \\
        \hline
        PPL & 7.06 & \bf{7.05} & 7.09 & 7.18 & 7.30 \\
        \hline
    \end{tabular}
\end{table}

\begin{table}[h]
    \centering
    \caption{$\lambda$ vs. Perplexity (PPL) for DeepSeekMoE ($\kappa=35$). Noise added to experts and dense modules.}
    \label{tab:lambda_ppl_simple}
    \begin{tabular}{c|cccccc}
        \hline
        $\lambda$ & 1.0 & 1.25 & \bf{1.5} & 1.75 & 2.0 & 2.25 \\
        \hline
        PPL & 7.75 & 7.52 & 7.51 & 7.58 & 7.70 & 7.85 \\
        \hline
    \end{tabular}
\end{table}

\clearpage

\section{Token Visualization of Neurons Corresponding to $\operatorname{MaxNNorm}$}

\begin{figure}[h]
    \centering
    \includegraphics[width=0.48\textwidth]{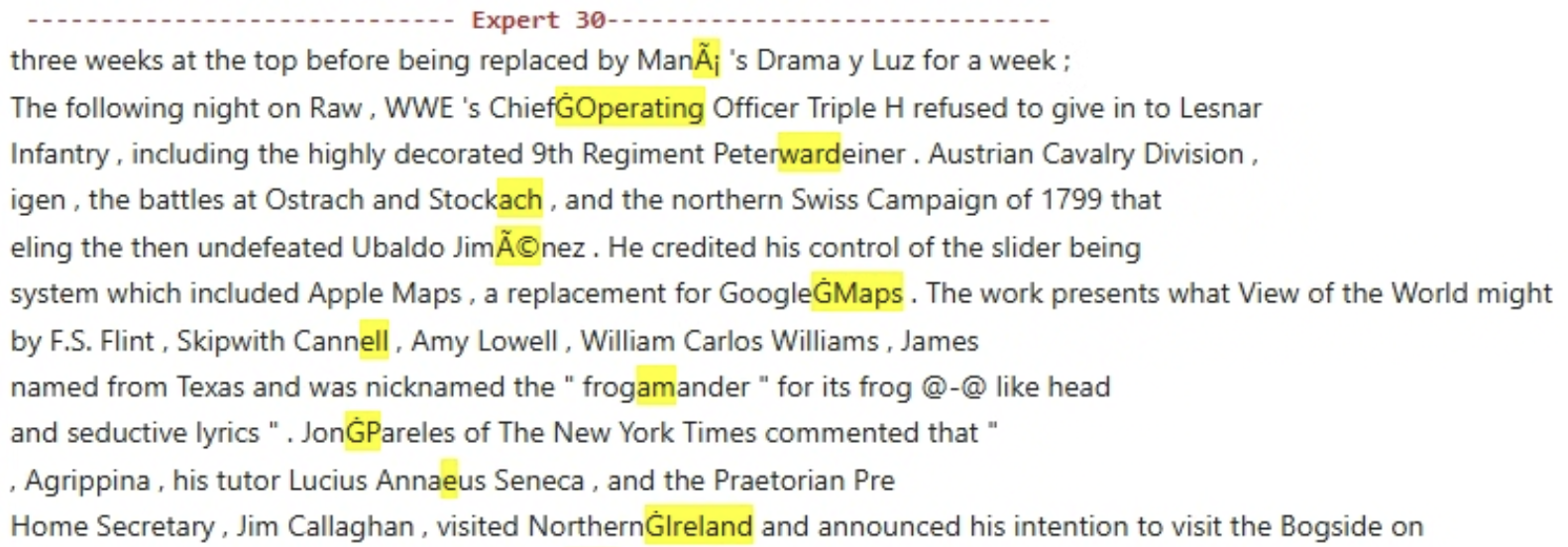}\hfill
    \includegraphics[width=0.48\textwidth]{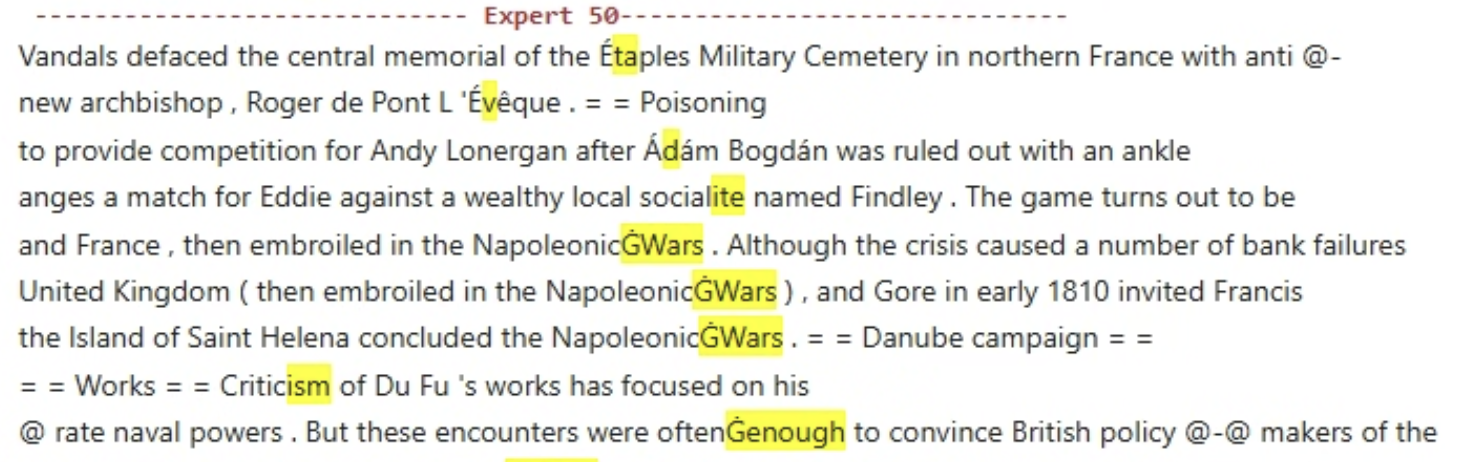}

    \vspace{0.3cm}

    \includegraphics[width=0.48\textwidth]{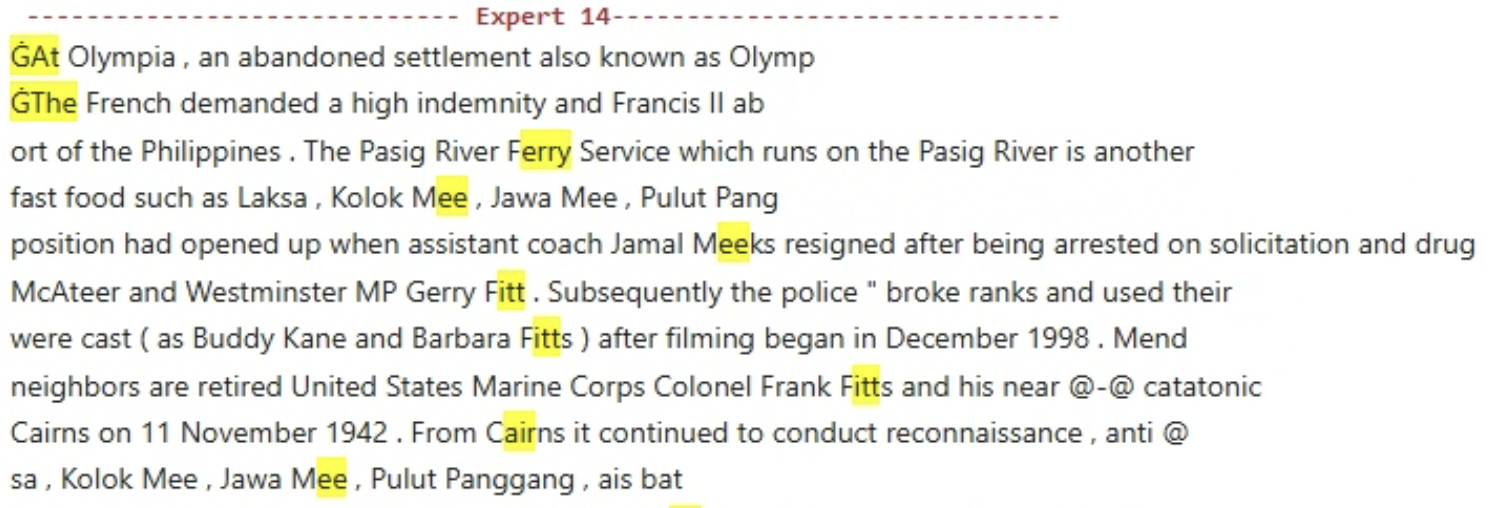}\hfill
    \includegraphics[width=0.48\textwidth]{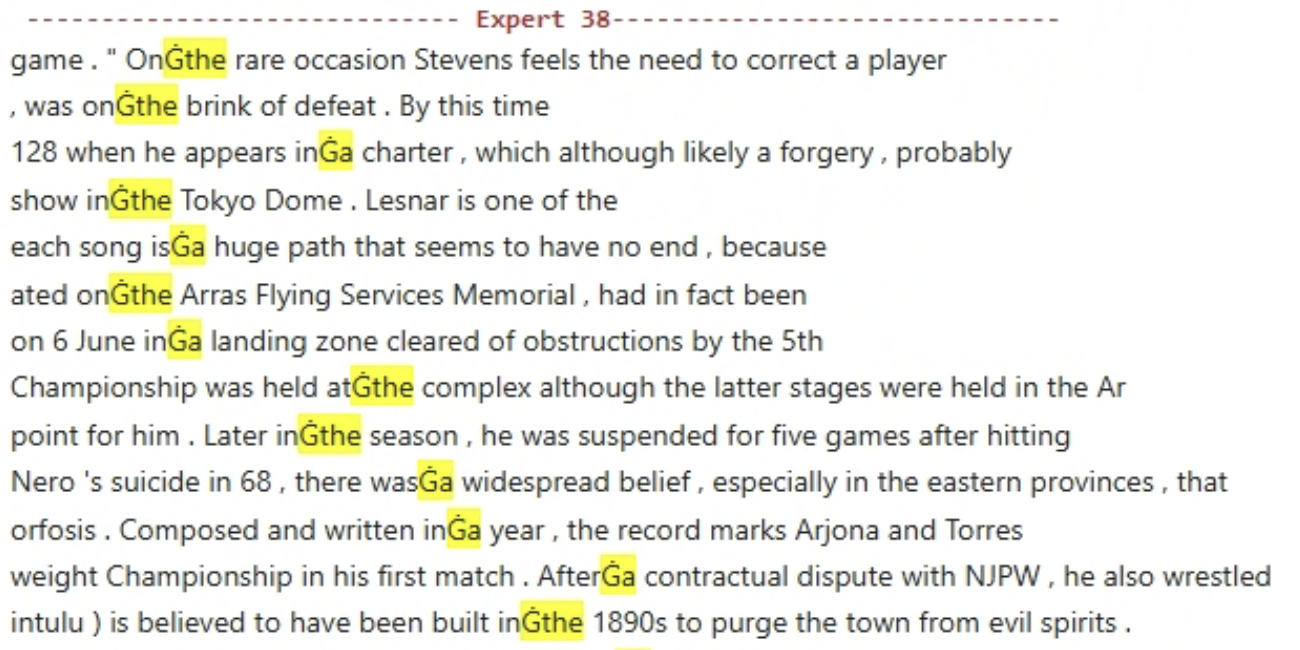}

    \vspace{0.3cm}

    \includegraphics[width=0.48\textwidth]{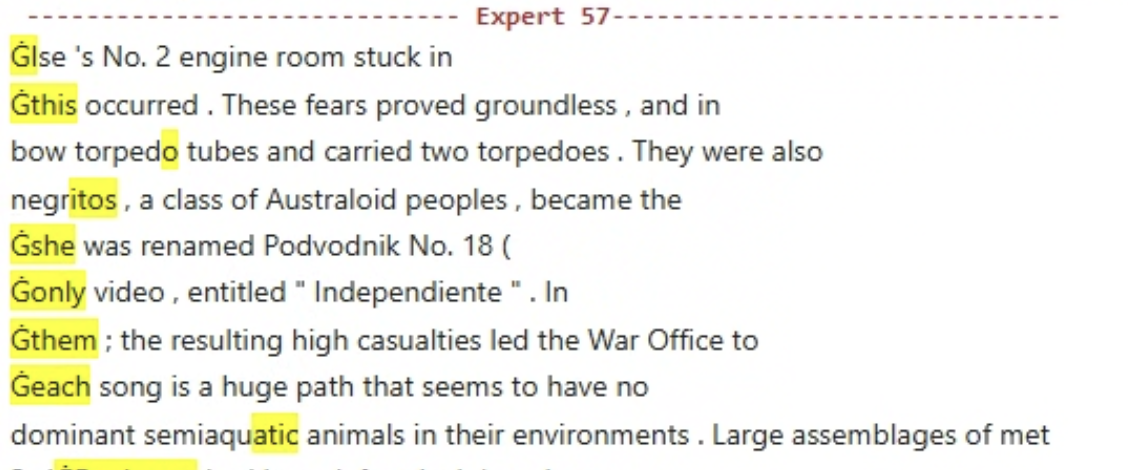}\hfill
    \includegraphics[width=0.48\textwidth]{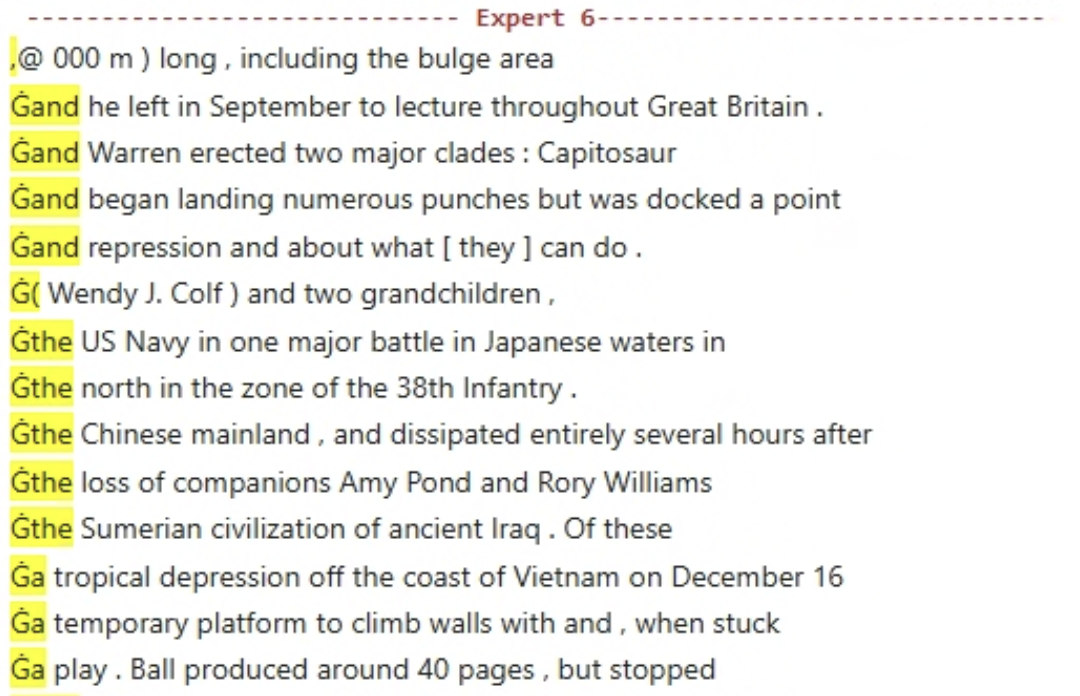}

    \caption{Token visualization of the first MoE block of OLMoE. Expert 30, 50, and 14 correspond to lowest three $MaxNNorm$ neurons over the up-projection matrices. Similarly, Expert 38, 57, and 6 correspond to the highest $MaxNNorm$ neurons over the up-projection matrices.}
    \label{fig:tok_vis}
\end{figure}

We visualize the top activating tokens of the neurons corresponding to maximum neuron norm for OLMoE. Figure \ref{fig:tok_vis} presents the results for the first layer of the model. Expert 30, 50, and 14 are the lowest $MaxNNorm$ experts for the up-projection matrices. Expert 38, 57, and 6 are highest $MaxNNorm$ experts for the up-projection layers. As we can see, the lowest neuron norm experts are activating by semantically less-frequent tokens (e.g., 'ach', 'ell', 'Ireland', 'ite', 'erry' etc.). On the other hand highest neuron norm experts are activating to semantically more-frequent tokens (e.g., 'the', 'a', 'this', 'them', 'each', 'and', etc.). This justify our intuition that higher $\operatorname{MaxNNorm}$ experts specialize to more-frequent tokens.

\clearpage
\section{Preliminaries}

For theoretical analysis, we adopt exactly the same setting as in \cite{chowdhury2026efficient}. We re-state the setting here.

As we theoretically analyze the expert-choice routing, for any training step $t$, equation (\ref{th_model}) can be re-written as,

\begin{equation}
    f^{(t)}(\mathbf{X})=\sum_{s=1}^kf_s^{(t)}(\mathbf{X}), \text{where } f_s^{(t)}(\mathbf{X})=a^{(s)}\sum_{j\in J_s^{(t)}(\mathbf{X})}G_j^{(s,t)}\sum_{r=1}^m\phi(\langle\mathbf{w}_r^{(s,t)},\mathbf{x}^{(j)}\rangle) 
\end{equation}

where, $J_s^{(t)}(\mathbf{X})$ is the set of tokens routed to expert $s$ at step $t$, $\mathbf{w}_r^{(s,t)}$ is the $r$-th neuron of $\mathbf{W}_{\mathrm{up}}^{(s)}$ at step $t$, and $\mathbf{x}^{(j)}$ is the $j$-th token of the sequence $\mathbf{X}$.

Tokens with top $l$ routing scores, i.e., the tokens of $\mathbf{X}$ with top $l$ indices of $[\mathbf{X}^\top\mathbf{\Sigma}]_{:,s}$ are routed to expert $s$.   

We consider $\phi=\operatorname{ReLU}(\cdot)$, i.e., the rectified linear unit. The routing weight of $j$-th token at expert $s$ can be calculated as,
\begin{equation}
G_j^{(s,t)} =
\begin{cases}
\dfrac{\exp(\langle \mathbf{w}_s^{(t)}, \mathbf{x}^{(j)} \rangle)}
{\sum_{i \in J_s^{(t)}(\mathbf{X})} \exp(\langle \mathbf{w}_s^{(t)}, \mathbf{x}^{(i)} \rangle)},
& j \in J_s^{(t)}(\mathbf{X}), \\[1.2ex]
0, & \text{otherwise}.
\end{cases}
\end{equation}
Here, $\mathbf{w}_s^{(t)}=\mathbf{\Sigma}_{:,s}^{(t)}$ is the routing vector corresponding to the expert $s$ at step $t$.

The model is trained to minimize the Hinge loss
\begin{equation}
    \hat{l}^{(t)}(\mathbf{X},y)=\max(1-yf^{(t)}(\mathbf{X}),0)
\end{equation} 
while the gradients are evaluated on
\begin{equation}
    l^{(t)}(\mathbf{X},y)=1-yf^{(t)}(\mathbf{X})
\end{equation}

For any input $(\mathbf{X},y)$, the gradient for $\mathbf{w}_r^{(s,t)}$ is evaluated as,
\begin{equation}
    \cfrac{\partial l^{(t)}(\mathbf{X},y)}{\partial \mathbf{w}_r^{(s,t)}}=-ya^{(s)}\sum_{j\in J_s^{(t)}(\mathbf{X})}G_j^{(s,t)}\mathbf{x}^{(j)}1_{\langle \mathbf{w}_r^{(s,t)},\mathbf{x}^{(j)}\rangle\ge0}
\end{equation}

The model is trained via Stochastic Gradient Descent algorithm (SGD) with batch size $B$. The expert weights are updated with learning rate $\eta_e$, and the routing parameters are updated with a learning rate $\eta_r$. The batch gradient for $\mathbf{w}_r^{(s,t)}$ is evaluated as,
\begin{equation}
    \cfrac{\partial l}{\partial \mathbf{w}_r^{(s,t)}}=\cfrac{1}{B}\sum_{\mathbf{X}\in\mathcal{B}_t}\cfrac{\partial l^{(t)}(\mathbf{X},y)}{\partial \mathbf{w}_r^{(s,t)}}
\end{equation}

\textbf{Notations:}
\begin{enumerate}
    \item $\Tilde{O}(\cdot)$ and $\Tilde{\Omega}(\cdot)$ hides the factor $\log(poly(d))$ with a sufficiently large polynomial $poly(\cdot)$
    \item With high probability (abbreviated as $w.h.p.$) refers to the probability $1-\cfrac{1}{poly(d)}$.
\end{enumerate}

\textbf{Definitions:}

For any $\mathbf{v}\in\mathcal{P}_r$, we define a complementary expert proficiency measure for the expert $s$ at time $t$ as,

$\bar{p}_\mathbf{v}^{(s,t)}:=\mathbb{P}\left[(\mathbf{X},y)\sim\mathcal{D}:\exists j\in J_s^{(t)} \text{ such that } \mathbf{x}^{(j)}=\mathbf{v}\big|\exists j\in[n] \text{ such that }\mathbf{x}^{(j)}=\mathbf{v}\right]$.

Without the loss of generality, we assume that for any $s\in S_\mathbf{v}$, $\bar{p}_{-\mathbf{v}}^{(s,0)}=O(1/d)$.

We define $C_n:=\max\left\{\left|\left|\mathbf{w}_r^{(s,0)}\right|\right|\right\}_{s\in[k], r\in[m]}$.

We define, $\gamma_{\mathbf{v}}:=\cfrac{\left|S_{\mathbf{v}}\right|}{\left|S_{+}\right|}$ for $\mathbf{v}\in\{\pm \mathbf{o}_1\}$ and $\gamma_{\mathbf{v}}:=\cfrac{\left|S_{\mathbf{v}}\right|}{\left|S_{-}\right|}$ for $\mathbf{v}\in\{\pm \mathbf{o}_2\}$.

Therefore, $\gamma=\max\{\gamma_{\mathbf{v}}\}_{\mathbf{v}\in\{\mathbf{o}_1,\mathbf{o}_2\}}$.

Without the loss of generality, we assume that $\forall \mathbf{v}\in\mathcal{P}_r, \gamma_{\mathbf{v}}=\Omega(1)$.

We define, $C_p:=\min\{\langle \mathbf{w}_s^{(0)},\mathbf{q}-\mathbf{q}^\prime\rangle\}_{s\in[k],\mathbf{q}\in\mathcal{P}\cup\{-\mathbf{o}_1,-\mathbf{o}_2\},\mathbf{q}^\prime\in\mathcal{P}\cup\{-\mathbf{o}_1,-\mathbf{o}_2\}\backslash\{\mathbf{q}\}}$.

We assume that $C_p>0$.

We assume that for any $\mathbf{v}\in\mathcal{P}_r$ and any $s\in S_\mathbf{v}$, $\frac{\left|\{r\in[m]:\langle \mathbf{w}_r^{(s,0)},\mathbf{v}\rangle\ge0\}\right|}{m}=\Omega(1)$,\\
and $\big||S_{+}|-|S_-|\big|=O(\sqrt{k})$.

We define,
\begin{itemize}
    \item $G_{\mathbf{v}}^{(s,t)}$: Routing weights of the token $x^{(j)}=\mathbf{v}$ for some $j\in[n]$ and $\mathbf{v}\in\mathcal{P}_r$ at expert $s$ and step $t$
    \item $G_{\mathbf{q}}^{(s,t)}$: Routing weights of the token $x^{(j)}=\mathbf{q}$ for some $j\in[n]$ and $\mathbf{q}\in\mathcal{P}\backslash\{\mathbf{o}_1,\mathbf{o}_2\}$ at expert $s$ and step $t$
    \item $l_\mathbf{q}^{(s,t)}:=\left|\{j\in J_s^{(t)}(\mathbf{X}):\mathbf{x}^{(j)}=\mathbf{q}\}\right|$, is the number of copies of the task-irrelevant vector $\mathbf{q}\in\mathcal{P}\backslash\{\mathbf{o}_1,\mathbf{o}_2\}$ in the set of top $l$ tokens for the input sequence $\mathbf{X}$ at expert $s$ and step $t$
\end{itemize}

\textbf{Components of the neurons' gradients.}

For any input $(\mathbf{X},y)\sim\mathcal{D}$, the gradient component of $\mathbf{w}_r^{(s,t)}$ along any task-relevant vector $\mathbf{v}\in\mathcal{P}_r$ and along any task-irrelevant vector $\mathbf{q}\in\mathcal{P}\backslash\{\mathbf{o}_1,\mathbf{o}_2\}$ at iteration $t$ are evaluated as follows:
\begin{equation}\label{eq_ec_c_q}
    \left\langle \cfrac{\partial l(\mathbf{X},y)}{\partial \mathbf{w}_r^{(s,t)}},\mathbf{q}\right\rangle=
\begin{cases}
    0 & \text{if $\langle \mathbf{w}_r^{(s,t)},\mathbf{q}\rangle<0$}\\\\
    0 & \text{if $\langle \mathbf{w}_r^{(s,t)},\mathbf{q}\rangle\ge0$ but, $\not\exists j \in J_s^{(t)}(\mathbf{X})$ s.t. $\mathbf{x}^{(j)}=\mathbf{q}$}\\\\
    -ya^{(s)} l_\mathbf{q}^{(s,t)}G_\mathbf{q}^{(s,t)} & \text{if $\langle \mathbf{w}_r^{(s,t)},\mathbf{q}\rangle\ge0$ and, $\exists j \in J_s^{(t)}(\mathbf{X})$ s.t. $\mathbf{x}^{(j)}=\mathbf{q}$}
\end{cases}
\end{equation}

\begin{equation}\label{eq_ec_c_v}
    \left\langle \cfrac{\partial l(\mathbf{X},y)}{\partial \mathbf{w}_r^{(s,t)}},\mathbf{v}\right\rangle=
\begin{cases}
    0 & \text{if $\langle \mathbf{w}_r^{(s,t)},\mathbf{v}\rangle<0$ but $\not\exists j \in J_s^{(t)}(\mathbf{X})$ s.t. $\mathbf{x}^{(j)}=-\mathbf{v}$}\\\\
    ya^{(s)}G_{-\mathbf{v}}^{(s,t)}(\mathbf{X}) & \text{if $\langle \mathbf{w}_r^{(s,t)},\mathbf{v}\rangle<0$ and $\exists j \in J_s^{(t)}(\mathbf{X})$ s.t. $\mathbf{x}^{(j)}=-\mathbf{v}$}\\\\
    0 & \text{if $\langle \mathbf{w}_r^{(s,t)},\mathbf{v}\rangle\ge0$ but $\not\exists j \in J_s^{(t)}(\mathbf{X})$ s.t. $\mathbf{x}^{(j)}=\mathbf{v}$}\\\\
    -ya^{(s)}G_\mathbf{v}^{(s,t)}(\mathbf{X}) & \text{if $\langle \mathbf{w}_r^{(s,t)},\mathbf{v}\rangle\ge0$ and $\exists j \in J_s^{(t)}(\mathbf{X})$ s.t. $\mathbf{x}^{(j)}=\mathbf{v}$}\\\\
\end{cases}
\end{equation}

\begin{lemma}[Lemma J.5(ii) of \cite{chowdhury2026efficient}]\label{lm2}
    For any expert $s\in S_\mathbf{v}$ such that $\mathbf{v}\in\mathcal\{\mathbf{o}_1,\mathbf{o}_2\}$, and $\forall \mathbf{q}\in\mathcal{P}\backslash\{\mathbf{o}_1,\mathbf{o}_2\}$, by selecting $\eta_r=O\left(\cfrac{\eta_eC_p}{ml^2C_n^2}\right)$ and $B=\Tilde{\Omega}\left(d^2\right)$, we can ensure that after $T^\prime=O\left(\cfrac{lC_2}{\alpha\eta_e}\right)$ steps, $p_\mathbf{v}^{(s,T^\prime)}\ge p_\mathbf{v}^{(s,0)}$ and, $\bar{p}_{-\mathbf{v}}^{(s,T^\prime)}\le \bar{p}_{-\mathbf{v}}^{(s,0)}$.
\end{lemma}

\begin{lemma}[Lemma J.6(ii) of \cite{chowdhury2026efficient}]\label{lm7}
    For any expert $s\in S_\mathbf{v}$ such that $\mathbf{v}\in\mathcal\{-\mathbf{o}_1,-\mathbf{o}_2\}$, and $\forall \mathbf{q}\in\mathcal{P}\backslash\{\mathbf{o}_1,\mathbf{o}_2\}$, by selecting $\eta_r=O\left(\cfrac{\eta_eC_p}{ml^2C_n^2}\right)$ and $B=\Tilde{\Omega}\left(d^2\right)$, we can ensure that after $T^\prime=O\left(\cfrac{lC_n}{(1-\alpha)\eta_e}\right)$ steps, $p_\mathbf{v}^{(s,T^\prime)}\ge p_\mathbf{v}^{(s,0)}$ and, $\bar{p}_{-\mathbf{v}}^{(s,T^\prime)}\le \bar{p}_{-\mathbf{v}}^{(s,0)}$.
\end{lemma}

\begin{lemma}[Lemma I.1(i) of \cite{chowdhury2026efficient}]\label{lm1_a}
    After $T=\Theta(l^2C_n\sqrt{\log l}/\alpha\eta_e\sqrt{C_p})$ steps of SGD with batch size $B=\Tilde{\Omega}(d^2)$ and $\eta_r=O\left(\cfrac{\eta_eC_p}{ml^2C_n^2}\right)$, for all $s\in S_{\mathbf{v}}$ and $\mathbf{v}\in\mathcal{P}_r=\{\pm\mathbf{o}_1,\pm\mathbf{o}_2\}$, we have $p_{\mathbf{v}}^{(s,T)}=1$, $\bar{p}_{-\mathbf{v}}^{(s,T)}=0$, and $\forall \mathbf{x}^{(j)}=\mathbf{v}$ for some $j\in[n]$, $G_j^{(s,T)}>\cfrac{1}{2}$.
\end{lemma}

\section{Proof of Lemma \ref{lm_m_1}}

\begin{lemma}[Full version of Lemma \ref{lm_m_1}]
    Suppose, the network given in (\ref{th_model}) is trained for
    $T=\Theta(l^2C_n\sqrt{\log l}/\alpha\eta_e\sqrt{C_p})$ steps with batch size $B=\Tilde{\Omega}(d^2)$ and $\eta_r=O(\cfrac{\eta_eC_p}{ml^2C_n^2})$. Then, for any $\mathbf{v}\in\{\mathbf{o}_1,\mathbf{o}_2\}$ and any expert $s,s^\prime$, such that $p_{\mathbf{v}}^{(s,T)}=1$ and $p_{\mathbf{-v}}^{(s^\prime,T)}=1$, we have
    \begin{equation}
        \operatorname{MaxNNScore}^{(s)}<\operatorname{MaxNNScore}^{(s^\prime)}
    \end{equation}
\end{lemma}

\begin{proof}
    From Lemma \ref{lm1_a}, for any $s\in S_{+}$, and for any $\mathbf{v}\in\{\mathbf{o}_1,-\mathbf{o}_1\}$, if $p_{\mathbf{v}}^{(s,T)}=1$, then $s\in S_{\mathbf{v}}$.

Similarly, for any $s\in S_{-}$, and for any $\mathbf{v}\in\{\mathbf{o}_2,-\mathbf{o}_2\}$, if $p_{\mathbf{v}}^{(s,T)}=1$, then $s\in S_{\mathbf{v}}$.

Now, WLOG let us assume $s\in S_{\mathbf{o}_1}$ and $s^\prime\in S_{-\mathbf{o}_1}$.

Then, using Lemma \ref{lm2}, $\forall r\in[m]$, $\forall t\le T$, $\left|\langle \cfrac{\partial l}{\partial \mathbf{w}_r^{(s,t)}},\mathbf{o}_1\rangle\right|\le\cfrac{\alpha}{2}+\Tilde{O}(\cfrac{1}{\sqrt{B}})$, which implies, for $B=\Tilde{\Omega}(d^2)$, $\forall r\in[m]$, $\left|\langle \mathbf{w}_r^{(s,T)},\mathbf{o}_1\rangle\right|=O(l^2C_n\sqrt{\log l}/\eta_e\sqrt{C_p})$.

On the other hand, $\forall r\in[m]$, $\forall t\le T$, if $\langle \mathbf{w}_r^{(s,t)},\mathbf{o}_2\rangle\ge 0$, then $\langle \cfrac{\partial l}{\partial \mathbf{w}_r^{(s,t)}},\mathbf{o}_2\rangle\ge 0$, and if $\langle \mathbf{w}_r^{(s,t)},\mathbf{o}_2\rangle\le 0$, then $\langle \cfrac{\partial l}{\partial \mathbf{w}_r^{(s,t)}},-\mathbf{o}_2\rangle\ge 0$. Therefore, $\left|\langle \mathbf{w}_r^{(s,T)},\mathbf{o}_2\rangle\right|=O(C_n)$.

Again, for any $\mathbf{q}\in\mathcal{P}\backslash\{\mathbf{o}_1,\mathbf{o}_2\}$, $\forall r\in[m]$, $\forall t\le T$, $\left|\langle \cfrac{\partial l}{\partial \mathbf{w}_r^{(s,t)}},\mathbf{q}\rangle\right|\le O(\cfrac{1}{d})+\Tilde{O}(\cfrac{1}{\sqrt{B}})$, which implies, $\left|\langle \mathbf{w}_r^{(s,T)},\mathbf{q}\rangle\right|=O(C_n)$.

Therefore, $\mathrm{MaxNNScore}^{(s)}=O(l^2C_n\sqrt{\log l}/\sqrt{C_p})$.

Now, as Lemma \ref{lm7} holds, $\forall r\in[m]$ such that $\langle \mathbf{w}_r^{(s^\prime,t)},-\mathbf{o}_1\rangle\ge0$,

$\forall t\ge \cfrac{\alpha T}{1-\alpha}$, $\langle \cfrac{\partial l}{\partial \mathbf{w}_r^{(s^\prime,0)}},-\mathbf{o}_1\rangle\le-\Omega(1-\alpha)+\Tilde{O}\left(\cfrac{1}{\sqrt{B}}\right)$.

Therefore, $\forall r\in[m]$ such that $\langle \mathbf{w}_r^{(s^\prime,0)},-\mathbf{o}_1\rangle\ge0$, $\langle \mathbf{w}_r^{(s^\prime,T)},-\mathbf{o}_1\rangle=\Omega((1-2\alpha)l^2C_n\sqrt{\log l}/\alpha\sqrt{C_p})$.

Therefore, $\mathrm{MaxNNScore}^{(s^\prime)}=\Omega((1-2\alpha)l^2C_n\sqrt{\log l}/\alpha\sqrt{C_p})$ which completes the proof.
\end{proof}

\section{Proof of Theorem \ref{th_1}}

\begin{theorem}[Full version of Theorem \ref{th_1}]
    Suppose, the network given in (\ref{th_model}) is trained for $T=\Theta(l^2C_n\sqrt{\log l}/\alpha\eta_e\sqrt{C_p})$ steps with batch size $B=\Tilde{\Omega}(d^2)$ and $\eta_r=O(\cfrac{\eta_eC_p}{ml^2C_n^2})$. Suppose, $\gamma$ fraction of the experts $s\in[k]$ in the trained model satisfies $p_{-\mathbf{v}}^{(s,T)}=1$ for $v\in\{\mathbf{o}_1,\mathbf{o}_2\}$. If
    \begin{equation}
        c\le c_A:=O(\cfrac{\alpha}{1-\alpha}\times\cfrac{1}{l^2\sqrt{d\log (kmld^2)}})
    \end{equation}
    with high probability the analog model has guaranteed generalization i.e., $\mathbb{P}[\forall(\mathbf{X},y)\sim\mathcal{D}:yf_{A}^{(T)}(\mathbf{X})>0]=1$. Furthermore, computing $\Gamma\ge\gamma$ fraction of the experts with top $\operatorname{MaxNNScore}$ in digital, if 
    \begin{equation}
        c\le c_H:= \cfrac{1-\alpha}{\alpha}c_A
    \end{equation}
    with high probability the heterogeneous model has guaranteed generalization i.e., $\mathbb{P}[\forall(\mathbf{X},y)\sim\mathcal{D}:yf_{H}^{(T)}(\mathbf{X})>0]=1$.
\end{theorem}

\begin{proof}

For any $r\in[m]$ of $s\in [k]$, the analog weights $\hat{\mathbf{w}}_r^{(s,T)}=\mathbf{w}_r^{(s,T)}+\Delta\mathbf{w}$ where, $\Delta\mathbf{w}$ is the programming error.

Now, for any $\mathbf{q}\in\mathcal{P}$, $\langle\Delta\mathbf{w},q\rangle\sim\mathcal{N}(0,\sigma^2)$, where $\sigma\le \sqrt{d}c_0\max(\mathbf{w}_r^{(s,T)})$.

Now, from Lemma \ref{lm1_a}, for any $s_1\in S_{\mathbf{o}_1}$, $p_{\mathbf{o}_1}^{(s_1,T)}=1$ and $\forall \mathbf{x}^{(j)}=\mathbf{o}_1$ for some $j\in[n]$, $G_j^{(s_1,T)}\ge\cfrac{1}{2}$.

Furthermore, $\sum_{r\in[m]}\langle \mathbf{w}_r^{(s,T)},\mathbf{o}_1\rangle=\Omega(mlC_2\sqrt{\cfrac{\log l}{C_p}})$. 

Therefore, for any $(\mathbf{X},y)\sim\mathcal{D}$ such that $\exists j\in[n]$ with $\mathbf{x}^{(j)}=\mathbf{o}_1$, 

$\sum_{s_1\in S_{\mathbf{o}_1}}f_{s_1}^{(T)}(\mathbf{X})=\Omega(\gamma_{o_1}mlC_n\sqrt{\cfrac{\log l}{C_p}})$.

On the other hand, from statement (i) of Lemma \ref{lm1_a}, for any $s_3\in S_{-o_1}$, $p_{o_1}^{(s_3,T)}=0$.

Therefore, for any $(x,y)\sim\mathcal{D}$ such that $\exists j\in[n]$ with $x^{(j)}=o_1$, 

$\sum_{s\in S_{+}}f_{s}^{(T)}(x)=\Omega(\gamma_{o_1}kmlC_2\sqrt{\cfrac{\log l}{C_p}})$.

Again, for any $q\in\mathcal{P}\backslash\{\mathbf{o}_1,\mathbf{o}_2\}$, $\forall s\in[k], \sum_{r\in[m]}\langle \mathbf{w}_r^{(s,T)},\mathbf{q}\rangle=O(mC_n)$, and $\forall s\in S_-, \sum_{r\in[m]}\langle \mathbf{w}_r^{(s,T)},\mathbf{o}_1\rangle=O(mC_n)$.

Therefore,  for any $(\mathbf{X},y)\sim\mathcal{D}$ such that $\exists j\in[n]$ with $\mathbf{x}^{(j)}=\mathbf{o}_1$, $f^{(T)}(\mathbf{X})=\Omega(\gamma_{o_1}kmlC_n\sqrt{\cfrac{\log l}{C_p}})-O(klmC_n)$, which implies for any $(\mathbf{X},y)\sim\mathcal{D}$ such that $\exists j\in[n]$ with $\mathbf{x}^{(j)}=\mathbf{o}_1$, $yf^{(T)}(\mathbf{X})>0$.\\

(\textbf{Condition 1}) Therefore, to ensure that for any $(\mathbf{X},y)\sim\mathcal{D}$ such that $\exists j\in[n]$ with $\mathbf{x}^{(j)}=\mathbf{o}_1$, $yf_A^{(T)}(\mathbf{X})>0$, we need $f^{(T)}(\mathbf{X})-f_A^{(T)}(\mathbf{X})= O(kmlC_n\sqrt{\cfrac{\log l}{C_p}})$.\\

(\textbf{Condition 2}) Similarly, to ensure that for any $(\mathbf{X},y)\sim\mathcal{D}$ such that $\exists j\in[n]$ with $\mathbf{x}^{(j)}=-\mathbf{o}_1$, $yf_A^{(T)}(\mathbf{X})>0$, we need $f^{(T)}(\mathbf{X})-f_A^{(T)}(\mathbf{X})= O(\cfrac{1-\alpha}{\alpha}kmlC_n\sqrt{\cfrac{\log l}{C_p}})$.\\

Now, for any $s_1\in S_{\mathbf{o}_1}$,  $\operatorname{MaxNNScore}^{(s_1)}=O(l^2C_n\sqrt{\cfrac{\log l}{C_p}})$, and $s_3\in S_{-\mathbf{o}_1}$ $\operatorname{MaxNNScore}^{(s_3)}=O(\cfrac{1-\alpha}{\alpha}l^2C_n\sqrt{\cfrac{\log l}{C_p}})$, which implies for any $r\in[m]$, $\max(\mathbf{w}_r^{(s_1,T)})\le O(l^2C_n\sqrt{\cfrac{\log l}{C_p}})$, and $\max(\mathbf{w}_r^{(s_3,T)})\le O(\cfrac{1-\alpha}{\alpha}l^2C_n\sqrt{\cfrac{\log l}{C_p}})$.\\

Similarly, for any $s_2\in S_{\mathbf{o}_2}$,  $\operatorname{MaxNNScore}^{(s_2)}=O(l^2C_n\sqrt{\cfrac{\log l}{C_p}})$, and $s_4\in S_{-\mathbf{o}_2}$ $\operatorname{MaxNNScore}^{(s_4)}=O(\cfrac{1-\alpha}{\alpha}l^2C_n\sqrt{\cfrac{\log l}{C_p}})$, which implies for any $r\in[m]$, $\max(\mathbf{w}_r^{(s_2,T)})\le O(l^2C_n\sqrt{\cfrac{\log l}{C_p}})$, and $\max(\mathbf{w}_r^{(s_4,T)})\le O(\cfrac{1-\alpha}{\alpha}l^2C_n\sqrt{\cfrac{\log l}{C_p}})$.\\

Therefore, to ensure that both condition 1 and condition 2 hold with probability $1-\cfrac{1}{d^2}$ over the randomness of programming noise, we need $c_A=O(\cfrac{\alpha}{1-\alpha}\times\cfrac{1}{l^2\sqrt{d\log(kmld^2)}})$.

Now as Lemma \ref{lm1_a} holds, placing $\gamma$ fraction of the experts with top $\operatorname{MaxNNScore}$ in digital, i.e., we need $c_H=O(\cfrac{1}{l^2\sqrt{d\log(kmld^2)}})$ to ensure condition 1 and 2 hold with probability at least $1-\cfrac{1}{d^2}$. 

\end{proof}

\end{document}